%% file: main.tex
\documentclass[10pt,twocolumn,letterpaper,dvipsnames,table]{article}
\usepackage[pagenumbers]{wacv} % To force page numbers, e.g. for an arXiv version

\usepackage{graphicx}
\usepackage{amsmath}
\usepackage{amssymb}
\usepackage{booktabs}

\usepackage[pagebackref,breaklinks,colorlinks]{hyperref}

\usepackage{xcolor}
\usepackage[export]{adjustbox}
\usepackage{colortbl}
\usepackage{graphicx}
\usepackage{booktabs}

\usepackage{array}      % nicer column spacing (optional)
\usepackage{makecell}
\usepackage{multirow}
\usepackage{subcaption}
\usepackage{threeparttable}
\usepackage{xcolor} % you already have this
\definecolor{best}{HTML}{F59E0B}      % amber
\definecolor{second}{HTML}{10B981}    % emerald

\newcommand{\best}[1]{\textcolor{best}{\textbf{#1}}}
\newcommand{\secondbest}[1]{\textcolor{second}{\textbf{#1}}}

\definecolor{streamrfgray}{gray}{0.45}
\definecolor{streamrfbg}{gray}{0.95} % optional light tint
\definecolor{egsbg}{HTML}{FFF7ED}
\definecolor{baselinebg}{gray}{0.97}

\newcommand{\egsrow}{\rowcolor{egsbg}}
\newcommand{\baselinerow}{\rowcolor{baselinebg}}
\newcommand{\xmark}{\textcolor{red!60!black}{\ensuremath{\times}}}
\newcommand{\cmark}{\textcolor{green!45!black}{\ensuremath{\checkmark}}}
\usepackage[accsupp]{axessibility}  % Improves PDF readability for those with disabilities.

\usepackage{graphicx}
\usepackage{amsmath}
\usepackage{amssymb}
\usepackage{amsthm}
\usepackage{booktabs}
\usepackage{array}
\usepackage{makecell}
\usepackage{multirow}
\usepackage{xcolor}
\usepackage{algorithm}
\usepackage{algpseudocode}
\usepackage[pagebackref,breaklinks,colorlinks]{hyperref}
\usepackage[capitalize]{cleveref}
\crefname{section}{Sec.}{Secs.}
\Crefname{section}{Section}{Sections}
\Crefname{table}{Table}{Tables}
\crefname{table}{Tab.}{Tabs.}
\Crefname{equation}{Eq.}{Eqs.}
\crefname{equation}{Eq.}{Eqs.}
\crefname{assumption}{Assumption}{Assumptions}
\Crefname{assumption}{Assumption}{Assumptions}
\crefname{proposition}{Proposition}{Propositions}
\Crefname{proposition}{Proposition}{Propositions}
\crefname{algorithm}{Alg.}{Algs.}
\Crefname{algorithm}{Algorithm}{Algorithms}

\newtheorem{proposition}{Proposition}
\newtheorem{assumption}{Assumption}

\usepackage[pagebackref]{hyperref}

\usepackage{orcidlink}

\usepackage[capitalize]{cleveref}
\crefname{section}{Sec.}{Secs.}
\Crefname{section}{Section}{Sections}
\Crefname{table}{Table}{Tables}
\crefname{table}{Tab.}{Tabs.}

\def\wacvPaperID{17} % *** Enter the WACV Paper ID here
\def\confName{WACV}
\def\confYear{2027}

\makeatletter
\AtBeginDocument{
\let\@oldmaketitle\@maketitle
\renewcommand{\@maketitle}{\@oldmaketitle
  \vspace{-15pt}
  \begin{center}
  \captionsetup{type=figure} % gives figure context outside a float

  \includegraphics[width=\linewidth]{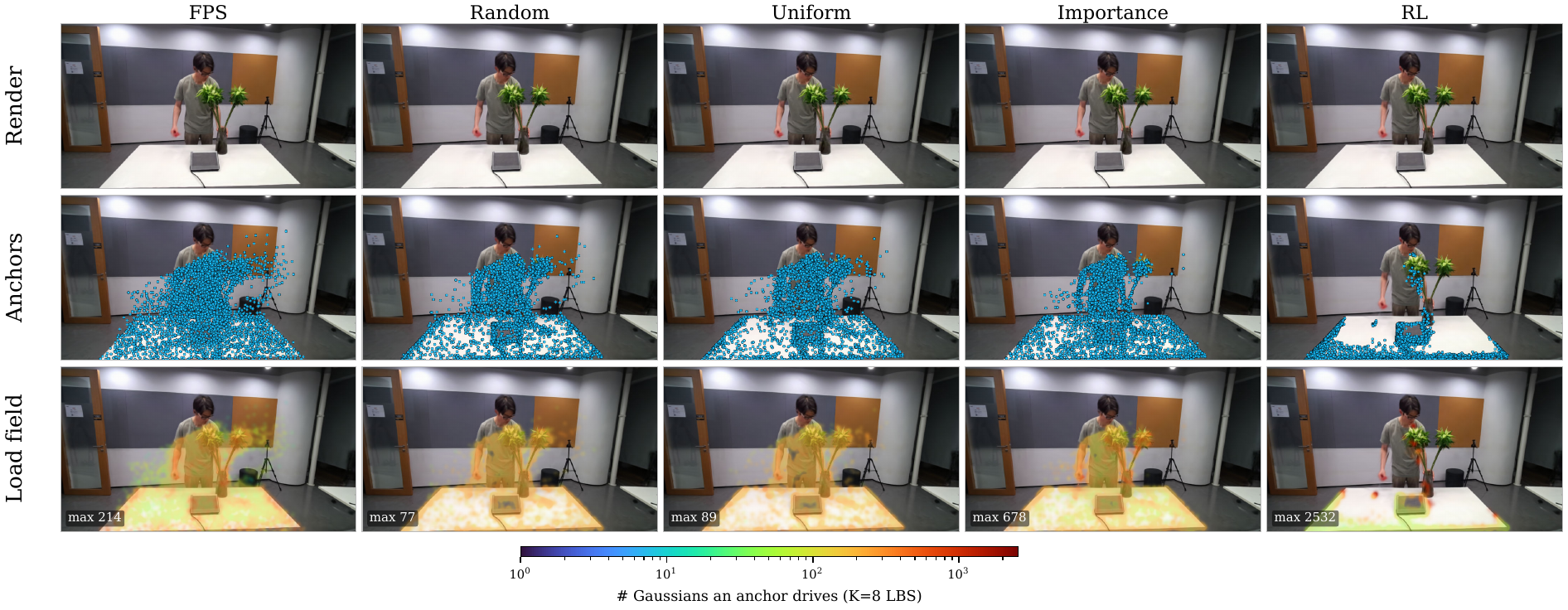}

\vspace{-3pt}
\caption{\small \textbf{Does it matter which Gaussians you pick in 4D Gaussian streaming?} For one held-out MeetingRoom \textit{Trimming} key-frame, each column is a selector and the rows are the refined render, the back-projected selected anchors, and the smooth $K{=}8$ LBS load field. The renders are identical, yet the per-anchor load differs sharply: the learned policy and the opacity-scale heuristic concentrate it on a few hot points while random and uniform spread it.}
\label{fig:fig_teaser}

  \end{center}
  \vspace{10pt}
 }
}
\makeatother

\begin{document}

\title{Does Anchor Selection Matter in 4D Gaussian Streaming?}

\author{Ashim Dahal \quad Rabab Abdelfattah \quad Nick Rahimi\\
School of Computing Sciences and Computer Engineering\\
University of Southern Mississippi, Hattiesburg, MS\\
{\tt\small \{ashim.dahal, rabab.abdelfattah, nick.rahimi\}@usm.edu}
}
\maketitle

\input{sections/0_abstract.tex}
\input{sections/1_introduction}
\input{sections/2_related_works}
\input{sections/3_methodology}
\input{sections/4_implementation}

\input{sections/5_exerpiments}
\input{sections/6_conclusions}
\input{sections/limitations.tex}
\clearpage

%%%%%%%%% REFERENCES
{\small
\bibliographystyle{ieee_fullname}
\bibliography{main}
}

\clearpage
\appendix
\section{Appendix: Supplementary Material}

\input{supplementary_materials.tex}
\end{document}

%% file: sections/0_abstract.tex
\begin{abstract}
Anchor-driven 4D Gaussian streaming methods such as Instant Gaussian Stream (IGS) update a dynamic scene each frame from a compact set of Gaussian anchors, chosen by default with Farthest Point Sampling (FPS) at a fixed budget of $8{,}192$. Because these anchors act as control points that drive the whole scene through linear blend skinning, the rule used to choose them ought to affect reconstruction quality. We test this by holding the IGS pipeline fixed and changing only the sampler, comparing FPS, random, uniform, an opacity-scale heuristic, and a learned policy across budgets and refinement settings on N3DV and MeetingRoom. At deployment budgets the sampler has no measurable effect: a cheap random or uniform sampler at $4{,}096$ anchors matches FPS@8192 within measurement error, the default budget is over-provisioned, and the result holds on a second backbone (3DGStream). The learned policy is mixed rather than consistently better: it can improve the N3DV validation set at tight budgets, but does not give a stable cross-dataset rule, and selection is never the bottleneck because refinement dominates runtime. We release our full sweep and evaluation protocol as a sampler benchmark.
\end{abstract}

%% file: sections/1_introduction.tex
\section{Introduction}
\label{sec:introduction}

Dynamic novel-view synthesis enables free-viewpoint video, telepresence, and interactive capture of moving scenes. 3D Gaussian Splatting represents a scene as explicit anisotropic Gaussians that rasterize efficiently \cite{3dgs_og}, and recent dynamic and streaming variants extend it to time-varying scenes \cite{4dgs_og,li2024spacetimegs,e4dgs,Yang2023Deformable3G,yan2025instant}. In the streaming setting the representation is updated as each frame arrives, and the central concern is making that update cheap without losing quality.

Anchor-driven methods such as Instant Gaussian Stream (IGS) update each frame from a compact subset of Gaussian anchors \cite{yan2025instant}. These anchors are not only the Gaussians that move. They act as the control points of a linear blend skinning deformation, where every Gaussian is displaced by a distance-weighted blend of its $K$ nearest anchors' predicted rigid transforms, so the anchor set drives the motion of the entire scene. By default IGS selects the anchors with Farthest Point Sampling (FPS) \cite{qi2017pointnetplusplus,gonzalez1985clustering} at a fixed budget of $8{,}192$.

Because the anchors steer the whole deformation, the rule used to place them should matter, and a better sampler, whether well-spread, content-aware, or learned, should give a better reconstruction. We set out to confirm this. Holding the IGS reconstruction and refinement pipeline fixed and changing only the sampler, we compare FPS, random selection, uniform voxel selection, an opacity-scale importance heuristic, and a learned policy across the anchor budgets and refinement settings the pipeline actually uses, on N3DV and MeetingRoom.

The expectation does not hold. At deployment budgets the fixed samplers have no measurable effect on quality. Their spread stays below the measurement noise we calibrate, and no fixed sampler is consistently best. What turns out to be over-provisioned is the budget, not the rule. Quality saturates by a few thousand anchors, so a plain random or uniform sampler at $4{,}096$ anchors matches FPS@8192 within measurement error and halves the budget at no cost. The same picture holds on a second, structurally different backbone, 3DGStream. Because local refinement dominates runtime, selection is never the bottleneck, and swapping the sampler at a fixed budget changes end-to-end time by about one percent. We also report a practical caveat for systems builders: an exact GPU port of FPS is slower than the optimized CPU routine IGS already uses, so the obvious GPU port does not help.

We include a learned policy as a case study rather than a proposed method, since it is the one sampler that behaves differently from the rest. Its behavior is not a clean win or a clean failure: it improves the N3DV validation set at tight budgets, but this does not carry over to the MeetingRoom scenes, and its MeetingRoom $4{,}096$-anchor parity point is a capped-set operating point rather than evidence of ranking. This is consistent with the rest of the study: when the choice of anchors barely affects quality, a learned scorer has little stable signal to exploit. Our study is limited to IGS-style anchor-driven streaming, and our broad absolute claims are anchored on the MeetingRoom scenes, where they sit in the normal dynamic Gaussian range.

Our contributions are:
\begin{itemize}
  \item \textbf{A controlled study of anchor selection.} Under one reconstruction code path we compare five samplers at four budgets and four refinement settings, and judge differences with confidence intervals over held-out frames rather than a hand-set threshold.
  \item \textbf{A leaner default.} We show that the FPS@8192 default is over-provisioned, and that a cheap sampler at $4{,}096$ anchors matches it within measurement error.
  \item \textbf{A cost breakdown.} We show that selection is not the runtime bottleneck, and that an exact GPU port of FPS is slower than the optimized CPU routine IGS already uses.
  \item \textbf{A released benchmark and a learned-sampler case study.} We release the sweep and protocol, and analyze a learned sampler whose gains on the N3DV validation set do not translate into a stable cross-dataset selection rule.
\end{itemize}

%% file: sections/2_related_works.tex
\section{Related Work}
\label{sec:related_work}

\noindent\textbf{Novel view synthesis and efficient scene representations.}
Novel view synthesis has been widely studied with implicit neural representations such as NeRF and dynamic variants that model time via deformation or conditioning \cite{mildenhall2021nerf,dnerf,park2021nerfies,Park2021HyperNeRF,nsff-dynamic}, supported by multi-view video benchmarks for free-viewpoint video \cite{neu3dv,technicolor,zju-mocap,zju-light}. Complementary approaches improve efficiency via factorized space--time fields \cite{Cao2023HexPlaneAF,fridovichkeil2023kplanes} and explicit point-based representations; in particular, 3D Gaussian Splatting (3DGS) enables real-time rendering with differentiable rasterization and has become a strong foundation for interactive view synthesis \cite{3dgs_og}, with ongoing work improving efficiency and scalability \cite{TuYing2025speede3dgs,HansonTuPUP3DGS,yuan2025differentiablehardware3dgs}.

\noindent\textbf{Dynamic Gaussian splatting.}
Dynamic extensions of splatting model motion and time-varying appearance using explicit Gaussian primitives \cite{4dgs_og,li2024spacetimegs,e4dgs,Yang2023Deformable3G}.
Recent work explores broader capture settings and additional priors, including multi-view consistency, physics, or hybrid representations \cite{Jiang2023Consistent4DC3,kim2025physgaia,oh2025hybrid,wang2025freetimegs,rotorgs,deng2025dynasplat,yin2025splat4d,li2025mtgs,zheng2024pku,lu2024diva,kirschstein2023nersemble,shapovalov2023replay}.
These methods can achieve strong quality, but many remain offline or do not explicitly target strict streaming budgets.

% methodology figure declared early (pushed up) so [!t] lands it atop page 3.
\begin{figure*}[!t]
  \centering
  \includegraphics[width=\linewidth]{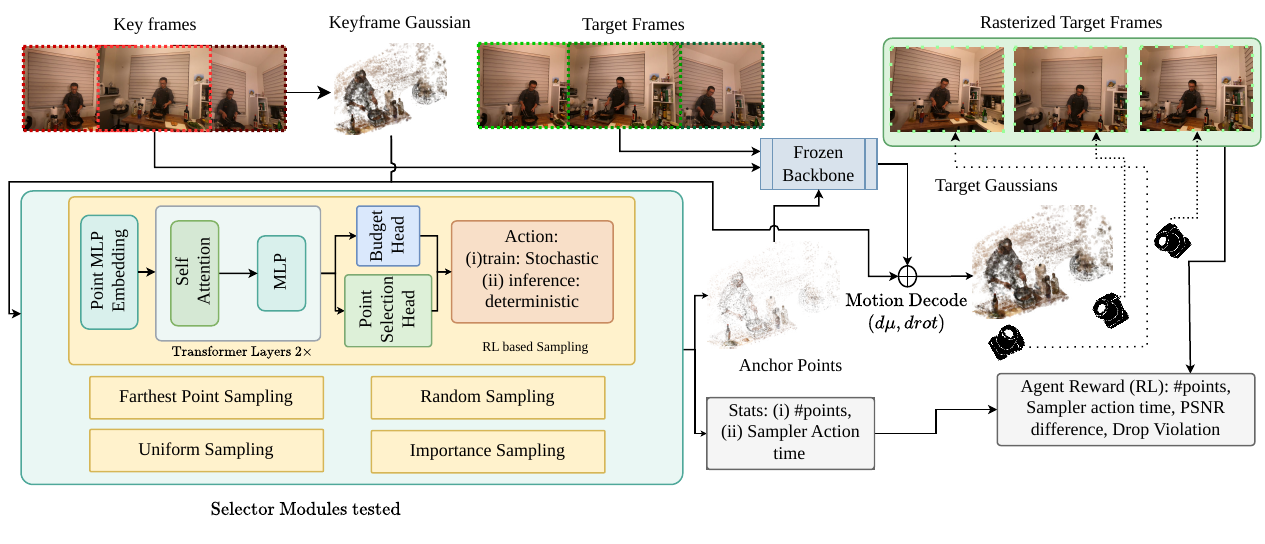}
  \caption{\textbf{Methodology of study.} We hold the IGS streaming stack fixed (candidate construction, graph building, reconstruction, rasterization, and high-quality local refinement) and vary only the anchor-selection rule at a matched budget $\kappa$. Five selectors run under this single code path: FPS (the IGS default), random, uniform voxel selection, an opacity-scale importance heuristic, and a learned contextual-bandit policy (RL), isolating the effect of the rule on quality and latency.}
  \label{fig:methodology}
\end{figure*}

\noindent\textbf{Streaming and online dynamic reconstruction.}
Streaming radiance fields have been studied via incremental learning and residual representations \cite{li2022streamrf,wang2023rerf}.
Gaussian streaming methods pursue low-latency updates either by training Gaussians on-the-fly \cite{sun20243dgstream} or by using generalizable feed-forward motion prediction driven by a compact set of anchors \cite{yan2025instant,li2025gifstream,streaminglod}.
IGS is particularly relevant as it standardizes an anchor-driven streaming pipeline and motivates plug-in improvements to anchor selection under fixed budgets \cite{yan2025instant}.

\noindent\textbf{Anchor selection under budget constraints.}
Anchor-driven streaming critically depends on selecting a subset of Gaussians under strict compute and point budgets.
Existing systems typically use hand-crafted heuristics, most notably farthest point sampling (FPS), to promote spatial coverage \cite{qi2017pointnetplusplus,yan2025instant}.
While learned selection under resource constraints has been explored elsewhere \cite{wu2018blockdrop,wang2018skipnet,he2018amc,dovrat2019learningtosample,lang2020samplenet}, applying it to anchor-based Gaussian streaming with \emph{joint budget+subset} constraints remains underexplored.
A learned, context-aware policy is one natural instance of this idea, and we include such a policy in our study, but as one selector among several rather than as a proposed method; our goal is to measure whether \emph{any} selection rule, learned or hand-crafted, changes the outcome within the IGS streaming interface \cite{yan2025instant}.

\noindent\textbf{Learning policies for resource allocation.}
Reinforcement learning and contextual bandits provide principled tools for allocating limited computation based on observed context \cite{suttonbarto2018rl,li2010contextualbandit,langford2007epochgreedy}.
Policy-gradient methods offer a general mechanism for optimizing stochastic selection policies \cite{williams1992reinforce}.
Attention-based set encoders further provide a natural architecture for scoring and selecting elements from unordered candidate sets \cite{vaswani2017attention,lee2019settransformer}.
The learned selector (RL) we study combines these perspectives; we evaluate it on equal footing with FPS and simple non-FPS samplers under matched budgets, and find that it provides no advantage in this setting (\cref{sec:postmortem}).

%% file: sections/3_methodology.tex
\section{Methodology}
\label{sec:methodology}

Our study isolates anchor selection as a controlled intervention inside an unchanged anchor-driven 4D Gaussian streaming pipeline (IGS) \cite{yan2025instant}. Only the selection rule changes. The reconstruction, rasterization, and local refinement are identical across all conditions, so quality and runtime differences are attributable to the selector rather than to a different renderer, motion model, or refinement procedure.

\subsection{Budgeted Anchor Selection}
\label{sec:methodology:setup}
For frame $t$, let the available Gaussian set be $\mathcal{G}_t=\{(\mathbf{x}_j,o_j,\mathbf{s}_j)\}_{j=1}^{N_t}$, with center $\mathbf{x}_j\in\mathbb{R}^3$, opacity $o_j$, and scale $\mathbf{s}_j\in\mathbb{R}^3$. We mask Gaussians to the scene bounding box to form the candidate set $\mathcal{C}_t=\{c_m\}_{m=1}^{M_t}$. This set is not density-equalized. Across our scenes its local Gaussian density varies by $15$ to $18\times$, and the four classical selectors (FPS, random, uniform, importance) operate on it directly, so the comparison among them tests intrinsic placement on a raw, non-uniform pool rather than on a pre-equalized one. The learned policy scores an internally voxel-deduplicated candidate set \cite{niessner2013voxelhashing} and takes the top-$\kappa$ anchors. For the MeetingRoom post-mortem, the deployed checkpoint has $M_{\max}{=}4096$ and is run at $\kappa{=}4096$, so that specific operating point returns the capped set; a larger-cap MeetingRoom checkpoint, $M_{\max}{=}16{,}384$, is used only in \cref{tab:postmortem} to test ranking below the cap. The N3DV validation rows use the N3DV checkpoint at each reported budget. A selector returns a subset $\Omega_t\subseteq\{1,\ldots,M_t\}$ of size $\kappa$, the anchor budget, and the selected anchors are passed to the unchanged IGS graph construction and refinement path. We sweep the four budgets $\kappa\in\{1024,2048,4096,8192\}$ for the quality grid; $\kappa{=}16{,}384$ appears only in the forward-latency microbenchmark (\cref{fig:fig_latency}) and the larger-cap post-mortem (\cref{tab:postmortem}). FPS@8192 is the IGS default and our per-dataset reference.

The anchors are not merely the Gaussians that move. IGS deforms the whole scene by linear blend skinning (LBS) from the anchors: each Gaussian $g$ is displaced by a distance-weighted blend of its $K{=}8$ nearest anchors' predicted rigid transforms,
\begin{align}
\mathbf{x}'_g=\sum_{a\in\mathcal{N}_K(g)} w_{ga}\,\big[R_a(\mathbf{x}_g-\mathbf{x}_a)+\mathbf{t}_a+\mathbf{x}_a\big],
\label{eq:lbs}
\end{align}
with blend weights $w_{ga}$ given by a softmax over negative neighbor distances, forming a partition of unity ($\sum_a w_{ga}{=}1$), exactly as in IGS \cite{yan2025instant}. The anchor set therefore parameterizes the entire motion field, so where the anchors are placed is a priori consequential: selection is a load-bearing modeling choice, not a peripheral preprocessing step. This is why its near-irrelevance (\cref{sec:invariance}) is a non-trivial finding. One consequence matters for capped learned checkpoints: ranking occurs below their cap, while the classical selectors draw from the full pool $M_t \gg \kappa$ at every budget we test.

\subsection{Selectors Under Test}
\label{sec:methodology:selectors}
All selectors share the bounding-box candidate front end and differ only in how they choose $\kappa$ anchors: the classical four from the full in-bbox pool $M_t$, and RL from its internally capped set.
\begin{itemize}
  \item \textbf{FPS} (IGS default): farthest point sampling for spatial coverage. To avoid a strawman, we use the same optimized bucketed kd-line FPS that IGS uses (\texttt{fpsample}'s \texttt{bucket\_fps\_kdline\_sampling}), an approximate near-linear FPS, not a naive $O(N\kappa)$ loop.
  \item \textbf{Random}: a uniformly random subset. It is stochastic, and its mean PSNR varies by at most $0.06$~dB across three seeds.
  \item \textbf{Uniform}: a deterministic uniform stride over the candidate ordering.
  \item \textbf{Importance}: multinomial sampling weighted by opacity times scale, a cheap content-aware heuristic.
  \item \textbf{RL (learned)}: a contextual-bandit policy that scores each candidate from a $13$-dimensional descriptor (geometric, opacity and scale, projected visual, and motion channels) with a point-wise MLP and a 2-layer Transformer set encoder \cite{vaswani2017attention,lee2019settransformer}, then takes the top-$\kappa$ scores. It is trained with a REINFORCE objective \cite{williams1992reinforce,li2010contextualbandit} against a refinement-aware reward and is exportable to TensorRT. RL receives no privileged status in this study; full training details, reward weights, and the warm-start schedule are in the supplement.
\end{itemize}
RL additionally has a budget head that could predict $\kappa$ per frame. We do not evaluate adaptive budgeting in this paper (all conditions use a fixed $\kappa$) and treat it as out of scope.

%% file: sections/4_implementation.tex
\section{Experimental Setup}
\label{sec:experiments}

\subsection{Datasets}
\label{subsec:data_and_splits}
We study two datasets, N3DV and MeetingRoom. Our broad absolute claims are anchored on \textbf{MeetingRoom} \cite{li2022streamrf} (\textit{Discussion} and \textit{Trimming}; \textit{Vrheadset} is excluded due to severe corruption in the released sequence), whose absolute PSNRs ($30$--$33$~dB) are in the normal dynamic-GS range and where we run the full budget$\times$refinement grid. We additionally use \textbf{N3DV} \cite{neu3dv} for two checks: the four-scene relative grid (\textit{Coffee Martini}, \textit{Cook Spinach}, \textit{Flame Salmon}, \textit{Flame Steak}) and the held-out validation-set metrics reported in \cref{tab:main_results}. The learned policy is trained only on N3DV. In the high-refinement N3DV configuration, key-frame test-time refinement adds a near selector-uniform offset ($+1.2$~dB to every selector alike; \cref{sec:results:keyframe}); because all comparisons are made within one fixed configuration and the input cameras are disjoint from the held-out evaluation cameras (\cref{sec:impl:eval}), this does not favor one selector over another.

\subsection{Matched-Budget, Single-Code-Path Protocol}
\label{sec:impl:eval}
Every condition runs through one reconstruction/refinement code path; only the selector and the budget $\kappa$ change. Evaluation is \textbf{held-out-camera}: the model receives a fixed set of input cameras (disjoint from the evaluation cameras) and renders the held-out cameras, scored against ground truth, across all streamed frames.

\noindent\textbf{Refinement strength is our primary axis.} Local refinement at evaluation is a key-frame--guided test-time optimization run every fifth frame for $r$ iterations. We sweep $r\in\{0,5,25,50\}$ to measure how the effect of the selection rule depends on refinement, in addition to the deployed high-quality setting.

\noindent\textbf{Key-frame refinement is selector-uniform (no method-favoring leakage).} At key frames the test-time optimization uses the captured views, which raises the key-frame PSNR by a roughly constant amount. We verified this offset is \emph{indistinguishable across selectors} (within CIs: $+1.20$~dB FPS, $+1.22$ RL, $+1.23$ random) and affects only $20\%$ of frames ($\approx{+}0.24$~dB on the average), so it never biases a cross-selector comparison; we additionally report steady-state (non-key-frame) PSNR.

\noindent\textbf{Metrics and normalization.} We report PSNR ($\uparrow$), LPIPS/DSSIM ($\downarrow$), end-to-end time per target frame, and selection-stage latency. For a method $m$, speedups and quality deltas are normalized to the same-dataset FPS@8192 reference:
\begin{align}
S_{\star}(m)&=\frac{T_{\star}(\mathrm{FPS@8192})}{T_{\star}(m)},\nonumber\\[2pt]
\Delta\mathrm{PSNR}(m)&=\mathrm{PSNR}(m)-\mathrm{PSNR}(\mathrm{FPS@8192}),
\label{eq:speedup_def}
\end{align}
for $\star\in\{\mathrm{e2e},\mathrm{sel}\}$. Because refinement dominates total runtime, selection-stage speedups do not translate linearly into end-to-end speedups.

\subsection{Evaluation Protocol}
\label{sec:impl:bootstrap}
We judge ties with a proper equivalence test. Per-frame PSNRs of adjacent streamed frames are correlated, so we use a moving-block bootstrap ($10^4$ resamples) whose block length is set from the data: the integrated autocorrelation time of per-frame PSNR is ${\approx}6$ frames on \textit{Discussion} and ${\approx}47$ on the more slowly varying \textit{Trimming}, so we take $L{=}15$ and $L{=}50$ respectively, both safely above the correlation length, with conclusions stable across $L\in[10,50]$ (supplement). Paired resampling on common frames gives a confidence interval on the PSNR gap between any two samplers, valid even for the deterministic ones (FPS, uniform, RL). On this interval we run two one-sided tests (TOST) against a smallest-effect-of-interest $\delta{=}0.25$~dB, a sub-perceptual margin: two samplers are \emph{equivalent} when the gap interval lies within $\pm\delta$. To calibrate the four-sampler spreads of \cref{sec:results:grid} we additionally bootstrap the best-minus-worst statistic under the null (four independent estimates of a single sampler), yielding the spread attributable to measurement noise alone. Random's mean PSNR varies by only $\le0.06$~dB across three seeds, well inside the equivalence margin.

\input{tables/main_results}

%% file: tables/main_results.tex
% Auto-generated by media/gen_table1.py from verified results.json (do not hand-edit numbers).
\begin{table*}[t]
\centering
\caption{\textbf{The whole study at a glance: at deployment budgets the fixed rules match FPS@8192 on both MeetingRoom scenes and the N3DV validation set; the budget, not the rule, is what matters.} Five selectors $\times$ four budgets. \best{Best} and \secondbest{second-best} per column are shaded (higher PSNR / speedup and lower LPIPS / DSSIM are better); cost is normalized to FPS@8192 on \textit{Discussion}.}
\label{tab:main_results}
\scriptsize
\setlength{\tabcolsep}{3.4pt}
\renewcommand{\arraystretch}{1.18}
\begin{tabular}{l ccc ccc ccc cc}
\toprule
& \multicolumn{3}{c}{\textbf{MeetingRoom Discussion}} & \multicolumn{3}{c}{\textbf{MeetingRoom Trimming}} & \multicolumn{3}{c}{\textbf{N3DV Val Set}} & \multicolumn{2}{c}{\textbf{Cost}} \\
\cmidrule(lr){2-4}\cmidrule(lr){5-7}\cmidrule(lr){8-10}\cmidrule(lr){11-12}
\textbf{Selector} & PSNR$\uparrow$ & LPIPS$\downarrow$ & DSSIM$\downarrow$ & PSNR$\uparrow$ & LPIPS$\downarrow$ & DSSIM$\downarrow$ & PSNR$\uparrow$ & LPIPS$\downarrow$ & DSSIM$\downarrow$ & E2E$\times$ & Smp$\times$ \\
\midrule
\multicolumn{12}{l}{\textit{Anchor budget} $\kappa = 8{,}192$\textit{\quad(deployment reference)}}\\
\baselinerow FPS \;\textit{(ref)} & \secondbest{30.932} & 0.1195 & \secondbest{0.0251} & \secondbest{32.033} & \best{0.1146} & \secondbest{0.0248} & \secondbest{31.698} & 0.1070 & 0.02309 & 1.00 & 1.00 \\
Random & \best{30.934} & 0.1192 & 0.0252 & 31.908 & 0.1155 & 0.0251 & \best{31.792} & \best{0.1056} & \best{0.02285} & 0.99 & 1.74 \\
Uniform & 30.884 & \secondbest{0.1192} & 0.0252 & 31.916 & 0.1154 & 0.0250 & 31.651 & \secondbest{0.1063} & \secondbest{0.02302} & 1.00 & 2.70 \\
Importance & 30.929 & \best{0.1191} & \best{0.0251} & \best{32.079} & \secondbest{0.1149} & \best{0.0246} & 31.582 & 0.1072 & 0.02322 & 0.99 & 2.38 \\
\egsrow RL & 30.525 & 0.1236 & 0.0264 & 31.571 & 0.1168 & 0.0254 & 31.447 & 0.1089 & 0.02368 & 0.99 & 0.77 \\
\midrule
\multicolumn{12}{l}{\textit{Anchor budget} $\kappa = 4{,}096$\textit{}}\\
FPS & \secondbest{30.925} & \secondbest{0.1194} & \best{0.0251} & 32.003 & \best{0.1137} & 0.0248 & 31.702 & 0.1068 & 0.02297 & 1.05 & 1.56 \\
Random & 30.877 & \best{0.1192} & 0.0252 & 31.985 & 0.1155 & \secondbest{0.0247} & \secondbest{31.713} & \secondbest{0.1060} & 0.02298 & 1.08 & 3.87 \\
Uniform & 30.887 & 0.1196 & 0.0253 & \secondbest{32.014} & 0.1151 & 0.0248 & 31.692 & 0.1065 & \secondbest{0.02295} & 1.07 & 4.43 \\
Importance & 30.835 & 0.1196 & 0.0253 & \best{32.056} & \secondbest{0.1146} & \best{0.0247} & 31.620 & 0.1065 & 0.02302 & 1.06 & 3.61 \\
\egsrow RL & \best{30.956} & 0.1195 & \secondbest{0.0251} & 31.931 & 0.1164 & 0.0250 & \best{32.501} & \best{0.0979} & \best{0.02121} & 1.08 & 2.25 \\
\midrule
\multicolumn{12}{l}{\textit{Anchor budget} $\kappa = 2{,}048$\textit{}}\\
FPS & \secondbest{30.900} & 0.1195 & \best{0.0252} & \best{32.074} & \best{0.1140} & \best{0.0246} & 31.646 & 0.1066 & 0.02298 & 1.09 & 2.34 \\
Random & 30.868 & \secondbest{0.1195} & 0.0253 & \secondbest{32.022} & 0.1156 & 0.0248 & 31.830 & \secondbest{0.1057} & \secondbest{0.02270} & 1.08 & 5.50 \\
Uniform & \best{30.937} & \best{0.1195} & \secondbest{0.0253} & 31.984 & 0.1152 & 0.0248 & \secondbest{31.845} & 0.1058 & 0.02282 & 1.09 & 2.90 \\
Importance & 30.749 & 0.1199 & 0.0256 & 31.967 & 0.1157 & 0.0248 & 31.614 & 0.1068 & 0.02307 & 1.09 & 4.33 \\
\egsrow RL & 30.557 & 0.1214 & 0.0260 & 31.997 & \secondbest{0.1150} & \secondbest{0.0247} & \best{33.675} & \best{0.0860} & \best{0.01928} & 1.09 & 2.25 \\
\midrule
\multicolumn{12}{l}{\textit{Anchor budget} $\kappa = 1{,}024$\textit{}}\\
FPS & \secondbest{30.860} & \secondbest{0.1199} & \secondbest{0.0253} & \best{32.164} & \best{0.1132} & \best{0.0244} & 31.699 & 0.1063 & 0.02302 & 1.10 & 3.12 \\
Random & 30.788 & 0.1203 & 0.0256 & 31.957 & 0.1157 & 0.0249 & 31.669 & 0.1061 & 0.02298 & 1.06 & 6.24 \\
Uniform & 30.777 & 0.1200 & 0.0255 & 31.894 & 0.1163 & 0.0250 & \secondbest{31.760} & \secondbest{0.1059} & \secondbest{0.02282} & 1.05 & 3.46 \\
Importance & \best{30.915} & \best{0.1196} & \best{0.0253} & \secondbest{32.050} & \secondbest{0.1152} & \secondbest{0.0246} & 31.684 & 0.1062 & 0.02295 & 1.07 & 3.33 \\
\egsrow RL & 29.805 & 0.1295 & 0.0289 & 31.278 & 0.1190 & 0.0265 & \best{33.846} & \best{0.0843} & \best{0.01891} & 1.10 & 2.45 \\
\bottomrule
\end{tabular}
\end{table*}

%% file: sections/5_exerpiments.tex
\section{Results}
\label{sec:invariance}

\subsection{Effect on Quality}
\label{sec:results:invariance}
\cref{tab:main_results} lays out the whole study in one place: five selectors at four budgets, on both MeetingRoom scenes and N3DV, with perceptual metrics and cost. The pattern is immediate. At every deployment budget the cheap rules sit on top of FPS@8192 across all three datasets while running a $2$--$6\times$ cheaper sampler, and the gaps that do appear track the \emph{budget}, not the rule. The rest of this section makes that reading precise.

At deployment budgets the samplers are interchangeable, and we can put a number on it. For each sampler we test whether its paired PSNR gap to FPS@8192 stays inside a sub-perceptual $\delta{=}0.25$~dB (two one-sided tests on the moving-block bootstrap; \cref{sec:impl:bootstrap}). Every cheap selector and the MeetingRoom learned-policy parity point pass (\cref{tab:invariance}): the widest paired gap to FPS@8192 is $0.06$~dB, smaller than the ${\sim}0.32$~dB that four \emph{identical} samplers produce from measurement noise alone (\cref{sec:results:grid}). LPIPS and DSSIM agree (cross-sampler spread $\le 0.0007$ and $\le 0.0004$). And quality saturates by $2$--$4$k anchors (\cref{fig:invariance_curve}): \textbf{a cheap random or uniform sampler at $4{,}096$ anchors is indistinguishable from FPS@8192 on both scenes} ($-0.055$ on \textit{Discussion}, $-0.064$ on \textit{Trimming}; importance closer still at $+0.019$; full Trimming frontier in the supplement). The budget was the thing to tune all along; the rule was not.

\begin{figure}[t]
  \centering
  \begin{subfigure}{0.40\linewidth}\centering
    \includegraphics[width=\linewidth]{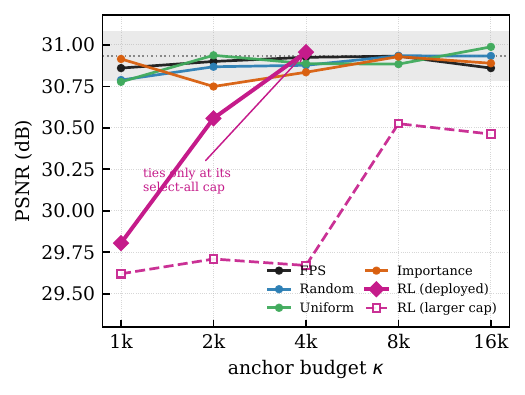}
    \caption{Quality saturates with budget.}\label{fig:invariance_curve}
  \end{subfigure}\hfill
  \begin{subfigure}{0.59\linewidth}\centering
    \includegraphics[width=\linewidth]{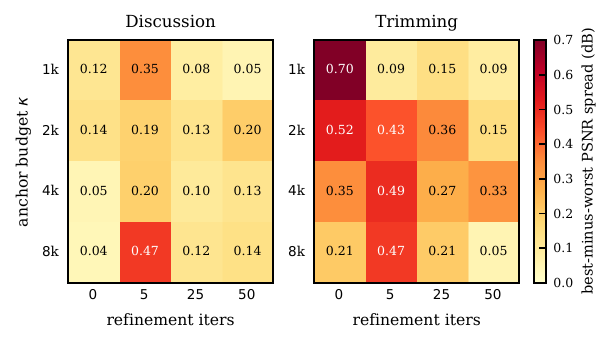}
    \caption{Best-minus-worst spread over the four standard samplers.}\label{fig:spread_grid}
  \end{subfigure}
  \caption{\textbf{The sampler barely affects quality.} \subref{fig:invariance_curve} PSNR vs.\ budget on \textit{Discussion}; the cheap rules coincide in the FPS@8192 noise band and saturate by ${\sim}4$k, while the learned policy ties only at the capped-set parity point. \subref{fig:spread_grid} best-minus-worst PSNR per budget$\times$refinement cell on both MeetingRoom scenes, below the calibrated noise floor.}
  \label{fig:quant}
\end{figure}

\begin{table}[!t]
\centering
\caption{\textbf{Equivalence test (paired CIs, TOST).} MeetingRoom \textit{Discussion}, $299$ held-out frames, moving-block bootstrap ($L{=}15$). $\Delta$ is PSNR minus FPS@8192 with its paired $95\%$ CI; \cmark{} marks statistical equivalence under TOST, \xmark{} a detectable gap. At $\kappa{\ge}4$k all cheap selectors are indistinguishable from FPS@8192; the $\kappa{=}1024$ gap is a budget effect.}
\label{tab:invariance}
\scriptsize
\setlength{\tabcolsep}{3.0pt}
\renewcommand{\arraystretch}{1.06}
\begin{tabular}{l c c c c}
\toprule
\textbf{Method} & $\kappa$ & \textbf{PSNR} $\uparrow$ & $\Delta$ \textbf{vs FPS@8192 (95\% CI)} & \textbf{Equiv.} \\
\midrule
FPS    & 8192 & 30.932 & \;\;0.000 (ref) & \cmark \\
FPS    & 4096 & 30.925 & $-0.007$ $[-0.115,+0.117]$ & \cmark \\
Random & 8192 & \secondbest{30.934} & $+0.002$ $[-0.082,+0.112]$ & \cmark \\
Random & 4096 & 30.877 & $-0.055$ $[-0.201,+0.096]$ & \cmark \\
Uniform & 4096 & 30.887 & $-0.045$ $[-0.182,+0.093]$ & \cmark \\
Importance & 8192 & 30.929 & $-0.003$ $[-0.137,+0.124]$ & \cmark \\
\egsrow RL & 4096 & \best{30.956} & $+0.024$ $[-0.080,+0.142]$ & \cmark \\
\bottomrule
\end{tabular}
\end{table}

\begin{figure*}[t]
  \centering
  \includegraphics[width=\linewidth]{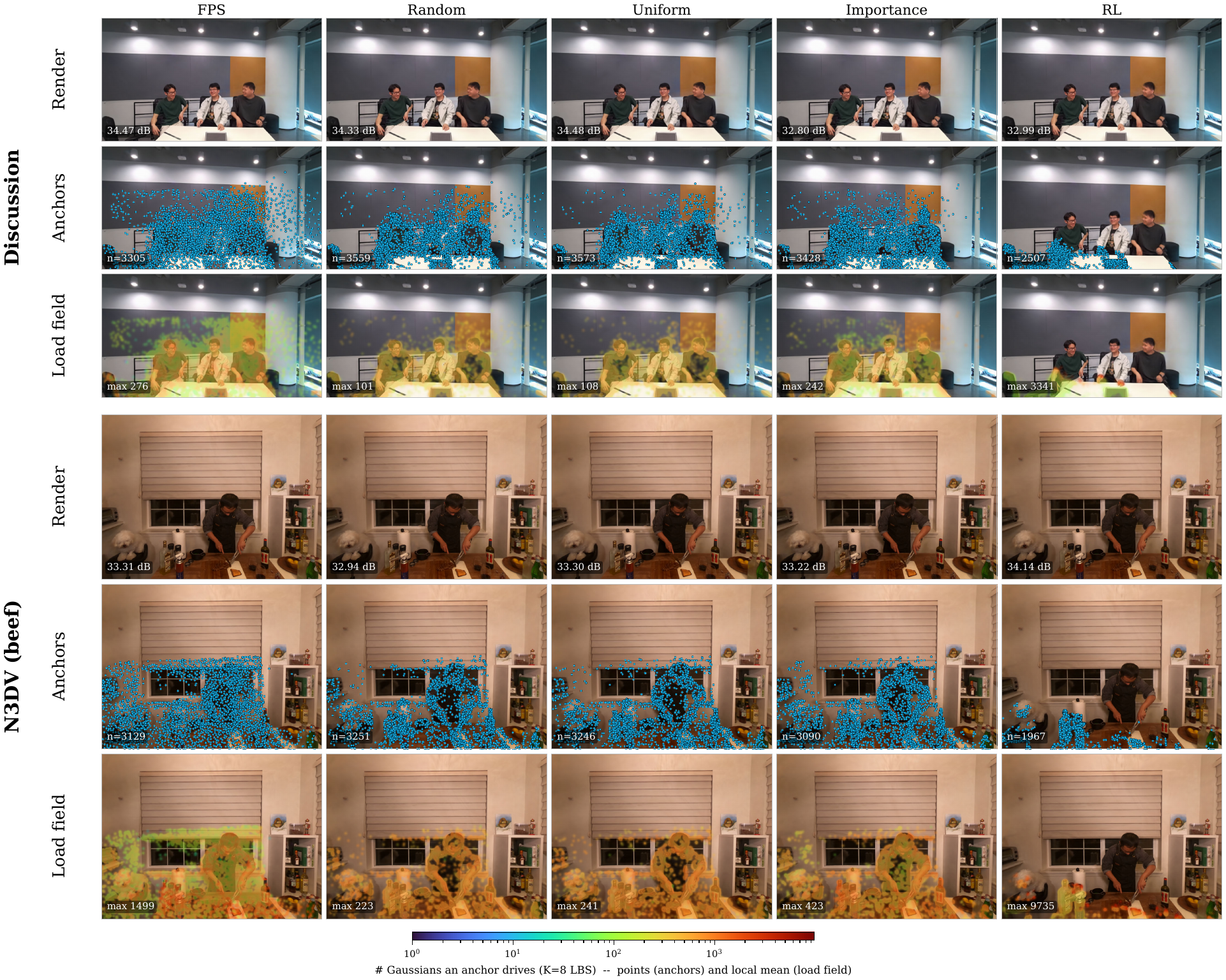}
  \caption{\textbf{Different anchors, identical reconstruction.} MeetingRoom \textit{Discussion} (top) and an N3DV scene (bottom); each block shows, per selector, the render, the back-projected anchors, and the $K{=}8$ LBS load field. Renders are indistinguishable; the per-anchor load differs sharply, the learned policy and the opacity-scale heuristic concentrating it while random and uniform spread it. \emph{More in the supplement.}}
  \label{fig:backproject}
\end{figure*}

\begin{figure}[!t]
  \centering
  \includegraphics[width=\linewidth]{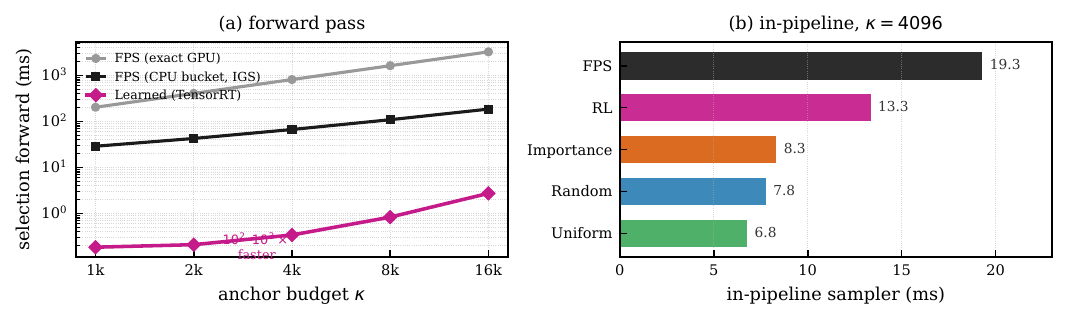}
\caption{\textbf{The cost of selection.} \emph{(a)} Latency vs.\ budget (A100): The TensorRT selector is $10^2$--$10^3\times$ faster than FPS; exact GPU FPS is slower than IGS's CPU FPS. \emph{(b)} Sampler-stage cost at $\kappa{=}4096$: FPS is costliest, followed by the learned scorer and cheap rules. Selection is a negligible slice of the ${\sim}0.4$\,s/frame.}
  \label{fig:fig_latency}
\end{figure}
\subsection{Scale of the Effect}
\label{sec:results:grid}
How general is this? We run the full budget$\times$refinement sweep on N3DV and both MeetingRoom scenes and plot the best-minus-worst PSNR over FPS/random/uniform/importance in every cell (\cref{fig:spread_grid}; full matrices in the supplement). The right yardstick for these spreads is a noise floor: how far apart do four \emph{identical} samplers land when each is estimated from the same noisy frames? We bootstrap exactly that: median $0.32$~dB on \textit{Discussion} and $0.86$~dB on \textit{Trimming} (larger because Trimming's PSNR is more autocorrelated, leaving fewer effective frames; \cref{sec:impl:bootstrap}). The four \emph{real} samplers disagree by \emph{less}: median $0.13$ ($\le 0.47$) on \textit{Discussion}, $0.30$ ($\le 0.70$) on \textit{Trimming}. The observed spread falls at the $6$th and $2$nd percentile of the null, so measurement noise alone produces a spread this large in $94\%$ and $98\%$ of resamples. Swapping the sampler moves quality less than re-rolling the measurement noise, so \textit{Trimming}'s larger spread is just its larger floor, not real sensitivity. N3DV agrees (median $0.06$). What little structure remains is small and unexploitable: it peaks at the tight-budget/no-refinement corner and fades as the budget grows, and no sampler wins twice. FPS tops \textit{Trimming} at $\kappa{=}4096$ ($+0.27$) and comes \emph{last} at $\kappa{=}8192$. No rule, fixed or learned, is worth preferring over a cheap one.

\subsection{Backbone Generalization}
\label{sec:results:secondbackbone}
IGS deforms the whole scene through an LBS blend of anchor transforms, so one might worry the sampler's irrelevance is special to that mechanism. We therefore repeat the test on \textbf{3DGStream}~\cite{sun20243dgstream}, a structurally different streaming backbone that updates the scene by gradient-driven densification through a neural transform cache rather than anchor-blended skinning. Here ``selection'' is which Gaussians the per-frame update acts on; the native rule is gradient-magnitude thresholding, and we swap in random and FPS at a \emph{matched budget} under one code path, exactly as before. On $50$ streamed frames of an N3DV scene, at quality in 3DGStream's normal range, the three rules are statistically indistinguishable (\cref{tab:secondbackbone}): a $0.02$~dB best-minus-worst spread, with every paired moving-block interval straddling zero. Selection invariance is not special to LBS, then: a second backbone with an entirely different update mechanism is just as indifferent to the sampler. We report this as a confirmatory check on one scene, not a full second grid.

\input{tables/secondbackbone}

\subsection{Key-frame Refinement Affects All Samplers Equally}
\label{sec:results:keyframe}
We confirm the result is not an artifact of the eval protocol. The per-key-frame test-time refinement (every $5$th frame) raises those frames' PSNR by a near-constant amount that is \emph{indistinguishable across selectors} (within CIs: $+1.20$~dB FPS, $+1.22$ RL, $+1.23$ random), contributing ${\sim}{+}0.24$~dB to the average and never biasing a cross-selector comparison. Steady-state (non-key-frame) PSNR preserves the same ordering and equivalences.

\section{Cost and Learned-Policy Analysis}
\label{sec:analysis}

\subsection{The Cost of Selection}
\label{sec:cost}
If the rule does not change quality, only its cost matters; three regimes are easy to conflate. \textbf{(i) Forward pass.} A learned selector exported to TensorRT runs in $0.19$--$2.74$~ms for $\kappa{=}1024$--$16{,}384$ (\cref{fig:fig_latency}; full table in the supplement), pool-independent and scaling only with the output budget. \textbf{(ii) Full in-pipeline stage.} With feature extraction the learned sampler costs $13.4$~ms at $\kappa{=}4096$, \emph{slower} than uniform ($6.8$), random ($7.8$), or importance ($8.3$), so it does not reduce sampler-stage cost. \textbf{(iii) FPS on the GPU.} An \emph{exact} GPU FPS (\texttt{torch\_cluster.fps}) is $7$--$18\times$ slower than IGS's \emph{approximate} bucketed CPU FPS (\cref{fig:fig_latency}); the obvious drop-in does not help.

End to end (\textit{Discussion}, per frame), swapping the selector at a matched budget moves time by ${\sim}1\%$ ($0.397$ vs.\ $0.401$\,s); the few-percent saving comes from \emph{halving} the budget, not the rule. \textbf{Deployment guidance:} use the cheapest selector (random/uniform) at ${\sim}4$k anchors and spend the budget on refinement, which dominates runtime.

\subsection{The Learned Policy}
\label{sec:postmortem}
We now examine the one learned policy (RL) we test, a contextual-bandit ranker trained on N3DV; we present this as a case study, not a general verdict on learned selection. Two facts are relevant.

\noindent\textbf{The MeetingRoom parity row is a capped-set operating point.} The deployed MeetingRoom checkpoint caps candidates at $M_{\max}=4096$ and is run at $\kappa=4096$, so top-$\kappa$ returns the capped set. \cref{tab:postmortem} makes this explicit: a larger-cap MeetingRoom checkpoint ($M_{\max}=16384$), evaluated at $\kappa=4096$, scores $29.7$~dB, which is $1.26$~dB below FPS@8192. The parity row is therefore not evidence that learned ranking transfers on MeetingRoom.

\noindent\textbf{Where it ranks, it does not transfer reliably.} Below the cap the learned scorer trails the cheap rules on \textit{Discussion} (\cref{tab:invariance}: $-0.38$~dB at $\kappa{=}2048$ and $-1.13$~dB at $\kappa{=}1024$ vs.\ FPS@8192), while improving the N3DV validation set at tight budgets (\cref{tab:main_results}). The policy is therefore a dataset-specific diagnostic, not a stable replacement rule. \cref{sec:results:grid} says why: when fixed-rule selection barely moves quality, a learned ranker has little stable signal to exploit outside the regime on which it was trained.

\begin{table}[!t]
\centering
\caption{\textbf{Post-mortem: the MeetingRoom parity row is capped-set.} At $M_{\max}{=}\kappa{=}4096$ the policy returns the capped set; with a larger cap it falls below FPS@8192 on \textit{Discussion}.}
\label{tab:postmortem}
\scriptsize
\setlength{\tabcolsep}{4.0pt}
\renewcommand{\arraystretch}{1.05}
\begin{tabular}{l c c c}
\toprule
\textbf{Configuration} & $M_{\max}$ & $\kappa$ & \textbf{PSNR} (\,$\Delta$ vs FPS@8192\,) \\
\midrule
FPS reference            & n/a   & 8192 & \secondbest{30.93} \;\;(0.00) \\
RL, $\kappa=M_{\max}$ (capped set) & 4096  & 4096 & \best{30.96} \;\;($+0.03$) \\
\egsrow RL, ranks below cap         & 16384 & 4096 & 29.67 \;\;($-1.26$) \\
\egsrow RL, ranks below cap         & 16384 & 8192 & 30.53 \;\;($-0.40$) \\
\bottomrule
\end{tabular}
\end{table}

%% file: tables/secondbackbone.tex
% Compact main-paper version of the 3DGStream cross-backbone table.
% \input this in Section 5.3 (Backbone Generalization). Self-contained (booktabs).
\begin{table}[t]
\centering
\caption{\textbf{Selection invariance holds on a second backbone (3DGStream).} Mean PSNR over $50$ matched streamed frames of an N3DV scene (filter-corrected init, normal quality range), matched budget, one code path. $\Delta$ and paired $95\%$ CI (moving-block bootstrap, $L{=}10$) are vs.\ the native gradient rule. A confirmatory single-scene check, not a second grid.}
\label{tab:secondbackbone}
\footnotesize
\setlength{\tabcolsep}{5pt}
\renewcommand{\arraystretch}{1.05}
\begin{tabular}{l c c c}
\toprule
\textbf{Selection rule} & \textbf{PSNR} $\uparrow$ & $\Delta$ vs.\ grad & \textbf{95\% CI} \\
\midrule
Gradient (native) & 32.69 & \;\;\;0.00 & n/a \\
Random            & 32.70 & $+0.01$ & $[-0.04,+0.04]$ \\
FPS               & 32.71 & $+0.02$ & $[-0.02,+0.06]$ \\
\midrule
\multicolumn{4}{l}{best$-$worst spread $=0.02$~dB; all paired intervals contain $0$} \\
\bottomrule
\end{tabular}
\end{table}

%% file: sections/6_conclusions.tex
\section{Conclusion}
\label{sec:conclusion}

We asked whether the anchor-selection rule matters in anchor-driven, IGS-style 4D Gaussian streaming, and answered it with a controlled, single-code-path comparison. The anchors act as control points that drive the whole scene through linear blend skinning, so one would expect their placement to matter, yet across N3DV and the two MeetingRoom scenes, all budgets ($1$k--$8$k), and all refinement settings ($0$--$50$), the gap between the best and worst standard sampler is small (median $0.06$~dB on N3DV, $0.13$ and $0.30$~dB on \textit{Discussion} and \textit{Trimming}) and below the measurement-noise floor we calibrate on every scene. No sampler is consistently best, the larger gaps appear only at the smallest budget with no refinement, and the one learned policy we studied is mixed: it improves the N3DV validation set at tight budgets, but does not transfer as a stable replacement rule. What is over-provisioned is the budget, not the rule: a cheap random or uniform sampler at $4{,}096$ anchors matches the FPS@8192 default within measurement error, and since refinement dominates runtime, selection is never the bottleneck (an exact GPU FPS is in fact $7$--$18\times$ slower than the optimized CPU routine, and the learned scorer is slower than random). The practical recommendation is therefore simple: use the cheapest sampler at roughly half the default budget, and spend the saved cost on refinement, which is what actually limits quality. We release the full sweep and evaluation protocol as a sampler benchmark.

%% file: sections/limitations.tex
\section*{Limitations}

\textbf{Scope of the backbone and datasets.} Our findings are scoped to anchor-driven (IGS-style) 4D Gaussian streaming: our full budget$\times$refinement grid is on IGS, and our broad absolute claims are anchored on MeetingRoom (\textit{Discussion}, \textit{Trimming}). We additionally confirm the finding on a second, structurally different backbone, 3DGStream, which is densification-driven rather than anchor-skinned (\cref{sec:results:secondbackbone}), where the selection rules are again statistically indistinguishable, as a single-scene check, not a full second grid; we do not claim cross-backbone universality, and replicating the complete grid on more streaming-GS designs remains the natural next step. In the high-refinement N3DV configuration, key-frame test-time refinement adds a \emph{selector-uniform} offset ($+1.2$~dB to every selector alike, ${\sim}{+}0.24$~dB on the average), so it cannot bias a cross-selector comparison within the fixed validation configuration.

\textbf{Scope of the learned-policy finding.} The learned-policy result is established for one contextual-bandit set-transformer and one content-aware heuristic (importance), in the heavy-refinement regime. We do not claim that \emph{no} learned selector could ever help; the policy improves the N3DV validation set at tight budgets, but does not give a stable replacement rule across the evaluated datasets. We also do not evaluate the policy's per-frame adaptive-budget head; it is out of scope here and left to future work.

\textbf{Cost measurement.} Our best-case latency isolates the network forward pass; the full in-pipeline stage additionally incurs feature extraction, which we report separately and do not attempt to fuse or optimize. End-to-end gains are small precisely because refinement dominates: a same-budget selector swap is about $1\%$, and halving the budget saves $4.3$--$8.0\%$, which is itself one of our findings rather than a limitation of the method under test.

%% file: supplementary_materials.tex
\noindent This supplement deepens the analysis of the experiments reported in
the main paper; it introduces no new datasets and no retuned method. It contains
(i)~a derivation that explains \emph{why} the anchor sampler barely changes the
render, and that predicts the single regime where a coverage-aware sampler does
help (\cref{sec:supp_theory}); (ii)~the algorithms behind our matched-budget
harness, the five selectors, and the moving-block bootstrap
(\cref{sec:supp_algos}); (iii)~the complete budget$\times$refinement spread grid
condensed into the teaser of the main paper (\cref{sec:supp_grid}); (iv)~the
per-selector quality frontiers and the Trimming equivalence test for both
MeetingRoom scenes (\cref{sec:supp_frontier}); (v)~the relative N3DV spread grid
and per-selector frontier (\cref{sec:supp_n3dv}); (vi)~the full cross-backbone
(3DGStream) selector comparison (\cref{sec:supp_3dgs}); (vii)~the bootstrap
procedure in full (\cref{sec:supp_bootstrap}); and (viii)~hyperparameters and
qualitative diagnostics (\cref{sec:supp_hyperparams,sec:supp_qualitative}). As in the main
paper, FPS@8192 is the per-scene reference and differences are judged against a
\emph{measured} confidence interval over held-out frames rather than an asserted
threshold.

% =====================================================================
\section{Why Anchor Placement Barely Changes the Render}
\label{sec:supp_theory}

The main paper's central empirical observation is that, across budgets and
refinement settings, swapping the anchor sampler changes reconstruction quality
by a small, scene-dependent amount and never in a way a particular rule can
exploit. We give a short argument for why this is expected for an LBS-driven
deformation, and why the one exception we observe (tight budget, no refinement)
is exactly where the argument is loosest.

\noindent\textbf{Setup.} IGS advances a frame by deforming every Gaussian center
$p\in\Omega$ with a linear-blend-skinning (LBS) map driven by the
\emph{selected} anchor set $S$ ($|S|=\kappa$):
\begin{equation}
\Phi_S(p)=\!\!\sum_{a\in\mathcal{N}_K^S(p)}\!\! w_a(p)\,\underbrace{\big[R_a\,(p-a)+a+t_a\big]}_{\displaystyle T_a(p)},
\label{eq:lbs}
\end{equation}
where $\mathcal{N}_K^S(p)$ are the $K$ nearest anchors of $p$ in $S$, the blend
weights $w_a(p)\ge0$ form a partition of unity ($\sum_a w_a(p)=1$; in IGS, a
softmax of negative neighbor distances), and $(R_a,t_a)$ is the rigid transform the (fixed) deformation
network predicts at anchor $a$. The sampler chooses \emph{which} points become
anchors; everything else in \cref{eq:lbs} is held fixed. We make three
assumptions, all mild for a smoothly moving scene.

\begin{assumption}[Smooth motion]\label{as:smooth}
The ground-truth per-point motion $\Phi^\star$ is $L_\Phi$-Lipschitz on $\Omega$.
\end{assumption}
\begin{assumption}[Anchor consistency]\label{as:consist}
Each anchor transform is correct at its own location,
$T_a(a)=\Phi^\star(a)$, and $T_a$ is $L_T$-Lipschitz in $p$.
\end{assumption}
\begin{assumption}[Matched coverage]\label{as:cover}
$S$ is a $\delta$-cover: every $p\in\Omega$ has a selected anchor within
$\delta$, so $\|p-a\|\le\rho\,\delta$ for all $a\in\mathcal{N}_K^S(p)$ and a
neighborhood constant $\rho$.
\end{assumption}

\begin{proposition}[Coverage, not identity, controls the deformation]
\label{prop:invariance}
Under \cref{as:smooth,as:consist,as:cover}, for every $p\in\Omega$,
\begin{equation}
\big\|\Phi_S(p)-\Phi^\star(p)\big\|\;\le\;(L_\Phi+L_T)\,\rho\,\delta .
\label{eq:bound}
\end{equation}
Consequently, two samplers that realize covers of radius $\delta$ and $\delta'$
produce deformations with
$\|\Phi_S(p)-\Phi_{S'}(p)\|\le(L_\Phi+L_T)\rho(\delta+\delta')$: the deformation
error is governed by the covering radius each sampler attains, precisely the
quantity a coverage-aware sampler such as FPS minimizes, rather than by any
finer property of \emph{which} points are chosen.
\end{proposition}

\begin{proof}
Using the partition of unity and \cref{as:consist},
\[
\Phi_S(p)-\Phi^\star(p)=\!\!\sum_{a\in\mathcal{N}_K^S(p)}\!\! w_a(p)\big[\,T_a(p)-\Phi^\star(p)\,\big].
\]
For each neighbor $a$, add and subtract $T_a(a)=\Phi^\star(a)$:
\[
T_a(p)-\Phi^\star(p)=\underbrace{[T_a(p)-T_a(a)]}_{\le L_T\|p-a\|}+\underbrace{[\Phi^\star(a)-\Phi^\star(p)]}_{\le L_\Phi\|p-a\|}.
\]
Both terms are bounded via \cref{as:consist,as:smooth}. Since
$\|p-a\|\le\rho\delta$ (\cref{as:cover}) and the weights are a convex
combination, \cref{eq:bound} follows. The two-sampler bound is the triangle
inequality. Only the covering radius $\delta$ enters: once two samplers attain
the same $\delta$ they are interchangeable to $O(\delta)$, and shrinking $\delta$
is exactly what a coverage-aware sampler optimizes, so FPS, which minimizes
$\delta$, is the strongest sampler precisely at small budgets where $\delta$ is large.
\end{proof}

\noindent\textbf{What the bound explains.} \Cref{prop:invariance} bounds the
\emph{deformation} error; we link that to PSNR only heuristically, a smaller
deformation error renders closer to ground truth, rather than through a formal
rendering-error analysis. With that caveat, the proposition reproduces every
qualitative feature of the measured grid (\cref{sec:supp_grid}):
\begin{itemize}\itemsep2pt
\item \textbf{Small and not exploitable.} The gap between any two samplers is
controlled by their covering radii, not by a clever ordering. Once two samplers
cover the scene comparably they are interchangeable to $O(\delta)$, so no fixed
or learned rule is consistently best, exactly what we observe.
\item \textbf{Shrinks with budget.} The covering radius of $\kappa$ anchors
scales as $\delta=\Theta(\kappa^{-1/2})$ (surface) to $\Theta(\kappa^{-1/3})$
(volume), so the bound, and thus the achievable cross-selector PSNR gap, decays
monotonically in $\kappa$.
\item \textbf{Scene-dependent.} The bound scales with $L_\Phi$, the motion
magnitude. \textit{Trimming}, with larger and faster motion, has a larger
$L_\Phi$ and therefore a looser bound and a larger measured spread than
\textit{Discussion} (\cref{tab:supp_grid}).
\item \textbf{Refinement compresses it.} Per-key-frame test-time optimization
re-fits the transforms to observations, directly reducing the residual
$T_a-\Phi^\star$ (effectively $L_T$) regardless of $S$, so the spread tends to
contract once any refinement is applied.
\end{itemize}

\noindent\textbf{The predicted exception.} The bound is loosest when $\delta$ is large
(small $\kappa$) and $L_\Phi$ is large (no refinement to correct it). There,
coverage \emph{quality} matters, and random sampling, which has the highest
covering-radius variance and can leave gaps, should be worst, while a sampler
that explicitly minimizes covering radius (FPS) should be best. This is exactly
the single cell where we measure a real gap: \textit{Trimming} at $\kappa{=}1024$,
$r{=}0$, FPS $20.76$~dB vs.\ random $20.06$~dB (\cref{tab:supp_grid}, top-left).
The theory thus predicts both halves of the empirical story: near-invariance
almost everywhere, and a coverage-aware sampler helping only in the
tight-budget/no-refinement corner that deployment never visits.

% =====================================================================
\section{Algorithms}
\label{sec:supp_algos}

We give the three procedures that define the study: the matched-budget
single-code-path harness (\cref{alg:harness}), the five selectors behind one
interface (\cref{alg:selectors}), and the moving-block bootstrap used for all
confidence intervals (\cref{alg:bootstrap}). Only the boxed selector call and the
budget $\kappa$ change between conditions; the reconstruction and refinement code
is byte-for-byte identical.

\begin{algorithm}[t]
\caption{Matched-budget, single-code-path evaluation}
\label{alg:harness}
\small
\begin{algorithmic}[1]
\Require stream $\{F_t\}$, selector $\sigma$, budget $\kappa$, refine iters $r$
\State init anchors from key frame; $\mathcal{M}\gets\varnothing$
\For{each frame $F_t$}
  \State $C_t\gets$ candidate Gaussians (pool, identical $\forall\sigma$)
  \State $S_t\gets \sigma(C_t,\kappa)$ \Comment{\emph{only} line that varies}
  \State $\{(R_a,t_a)\}\gets \textsc{DeformNet}(S_t)$ \Comment{fixed weights}
  \State $\hat{F}_t\gets \textsc{LBS-Render}(S_t,\{(R_a,t_a)\})$ \Comment{\cref{eq:lbs}}
  \If{$t \bmod 5 = 0$} \Comment{key frame}
    \State $\hat{F}_t\gets \textsc{Refine}(\hat{F}_t, r\text{ iters})$
  \EndIf
  \State $\mathcal{M}\gets\mathcal{M}\cup\{\textsc{PSNR}(\hat{F}_t,F_t^{\text{held-out}})\}$
\EndFor
\State \Return per-frame metrics $\mathcal{M}$
\end{algorithmic}
\end{algorithm}

\begin{algorithm}[t]
\caption{Selectors under a common interface $\sigma(C,\kappa)$}
\label{alg:selectors}
\small
\begin{algorithmic}[1]
\Function{FPS}{$C,\kappa$} \Comment{IGS default; bucketed kd-line}
  \State \Return farthest-point subset of size $\kappa$
\EndFunction
\Function{Random}{$C,\kappa$}
  \State \Return $\kappa$ points drawn uniformly without replacement
\EndFunction
\Function{Uniform}{$C,\kappa$} \Comment{voxel-stratified}
  \State voxelize $C$; draw $\propto$ per-voxel occupancy to size $\kappa$
\EndFunction
\Function{Importance}{$C,\kappa$} \Comment{content-aware heuristic}
  \State score $s_i\gets \mathrm{opacity}_i\times\mathrm{scale}_i$; \Return top-$\kappa$
\EndFunction
\Function{RL}{$C,\kappa$} \Comment{learned policy}
  \State $C'\gets$ truncate $C$ to cap $M_{\max}$
  \State $z\gets\textsc{Transformer}(\textsc{PointMLP}(C'))$
  \State \Return top-$\kappa$ by score $z$ \Comment{$\kappa{=}M_{\max}\!\Rightarrow$ capped set}
\EndFunction
\end{algorithmic}
\end{algorithm}

\begin{algorithm}[t]
\caption{Moving-block bootstrap CI for a paired PSNR gap}
\label{alg:bootstrap}
\small
\begin{algorithmic}[1]
\Require per-frame PSNR $a_{1:n}$ (method $A$), $b_{1:n}$ ($B$), block $L{=}15$, $N{=}10^4$
\State $d_i\gets a_i-b_i$ for $i=1{:}n$ \Comment{pair on common frames}
\For{$j=1$ \textbf{to} $N$}
  \State tile blocks of length $L$ from $\{d_i\}$ until $\ge n$ samples; truncate to $n$
  \State $\bar d^{(j)}\gets$ mean of the resampled sequence
\EndFor
\State \Return $\big(\bar d,\ [\,Q_{2.5},Q_{97.5}\,]\text{ of }\{\bar d^{(j)}\}\big)$
\Statex \Comment{$A,B$ indistinguishable if the interval contains $0$}
\end{algorithmic}
\end{algorithm}

% =====================================================================
\section{The Full Budget$\times$Refinement Grid}
\label{sec:supp_grid}

\cref{tab:supp_grid} gives the per-cell best-minus-worst PSNR spread over the
four standard samplers (FPS, random, uniform, importance) for every budget and
refinement setting on both MeetingRoom scenes, the data condensed into the
teaser heatmap of the main paper. \textit{Discussion} is near-invariant
throughout (median $0.13$~dB, max $0.47$); \textit{Trimming} is the most
sensitive scene (median $0.30$, max $0.70$), with the largest gap at the
tight-budget/no-refinement corner predicted by \cref{prop:invariance}. The
spread does not fall off perfectly monotonically, the partial-refinement
($r{=}5$) column is the noisiest, but the smallest spreads consistently occur at
the largest budgets, and no sampler is the best across cells (\cref{tab:supp_frontier}).

\begin{table}[t]
\centering
\caption{\textbf{Cross-selector PSNR spread (dB), full grid.} Best-minus-worst
over FPS/random/uniform/importance per cell; $299$ held-out frames per condition.
Rows: budget $\kappa$; columns: refinement iterations $r$.}
\label{tab:supp_grid}
\footnotesize
\setlength{\tabcolsep}{4.5pt}
\renewcommand{\arraystretch}{1.05}
\begin{tabular}{l cccc}
\toprule
$\kappa\;\backslash\;r$ & $0$ & $5$ & $25$ & $50$ \\
\midrule
\multicolumn{5}{l}{\emph{Discussion} \;(median $0.13$, max $0.47$)}\\
$1024$ & 0.117 & 0.347 & 0.083 & 0.050 \\
$2048$ & 0.138 & 0.191 & 0.127 & 0.202 \\
$4096$ & 0.046 & 0.196 & 0.100 & 0.125 \\
$8192$ & 0.040 & 0.471 & 0.122 & 0.140 \\
\midrule
\multicolumn{5}{l}{\emph{Trimming} \;(median $0.30$, max $0.70$)}\\
$1024$ & \textbf{0.697} & 0.093 & 0.145 & 0.086 \\
$2048$ & 0.522 & 0.430 & 0.355 & 0.151 \\
$4096$ & 0.346 & 0.495 & 0.272 & 0.332 \\
$8192$ & 0.208 & 0.472 & 0.207 & 0.050 \\
\bottomrule
\end{tabular}
\end{table}

% =====================================================================
\section{Per-Selector Quality Frontiers}
\label{sec:supp_frontier}

\cref{tab:supp_frontier} lists every selector at every budget for both
MeetingRoom scenes, with $\Delta$ measured against that scene's FPS@8192
reference. Two facts stand out. First, on both scenes the cheap samplers track
FPS to within the measurement floor for $\kappa\ge2048$; \emph{Trimming} even
prefers FPS@4096 ($+0.27$~dB) while ranking FPS@8192 below random and importance,
confirming that no single budget--sampler pair dominates. Second, the learned
policy (RL, \emph{Discussion}) is at parity only at its capped-set operating
point ($\kappa{=}M_{\max}{=}4096$); with smaller budgets it falls to $-0.38$
and $-1.13$~dB, the post-mortem of the main paper. RL is shown shaded, not
highlighted: its MeetingRoom parity does not transfer into a stable ranking
rule.

\begin{table}[t]
\centering
\caption{\textbf{Per-selector frontiers, MeetingRoom.} PSNR and
$\Delta$ vs.\ the scene's FPS@8192 reference; $299$ held-out frames.
\textit{Discussion} at the deployed protocol; \textit{Trimming} at full
refinement ($r{=}50$). RL rows are its learned-policy operating points.}
\label{tab:supp_frontier}
\scriptsize
\setlength{\tabcolsep}{3.2pt}
\renewcommand{\arraystretch}{1.04}
\begin{tabular}{l c c c | c c c}
\toprule
& \multicolumn{3}{c|}{\textbf{Discussion}} & \multicolumn{3}{c}{\textbf{Trimming}}\\
\textbf{Method} & $\kappa$ & PSNR & $\Delta$ & $\kappa$ & PSNR & $\Delta$ \\
\midrule
FPS    & 8192 & 30.932 & \;0.000 & 8192 & 31.946 & \;0.000 \\
FPS    & 4096 & 30.925 & $-0.007$ & 4096 & 32.214 & $+0.268$ \\
FPS    & 2048 & 30.900 & $-0.032$ & 2048 & 32.041 & $+0.095$ \\
FPS    & 1024 & 30.860 & $-0.072$ & 1024 & 31.980 & $+0.034$ \\
\midrule
Random & 8192 & 30.934 & $+0.002$ & 8192 & 31.994 & $+0.048$ \\
Random & 4096 & 30.877 & $-0.055$ & 4096 & 31.882 & $-0.064$ \\
Random & 2048 & 30.868 & $-0.064$ & 2048 & 32.000 & $+0.054$ \\
Random & 1024 & 30.788 & $-0.144$ & 1024 & 31.967 & $+0.021$ \\
\midrule
Uniform & 8192 & 30.884 & $-0.048$ & 8192 & 31.972 & $+0.026$ \\
Uniform & 4096 & 30.887 & $-0.045$ & 4096 & 31.894 & $-0.052$ \\
Uniform & 2048 & 30.937 & $+0.005$ & 2048 & 32.109 & $+0.163$ \\
Uniform & 1024 & 30.777 & $-0.155$ & 1024 & 32.053 & $+0.107$ \\
\midrule
Import. & 8192 & 30.929 & $-0.003$ & 8192 & 31.996 & $+0.050$ \\
Import. & 4096 & 30.835 & $-0.097$ & 4096 & 31.965 & $+0.019$ \\
Import. & 2048 & 30.749 & $-0.183$ & 2048 & 31.958 & $+0.012$ \\
Import. & 1024 & 30.915 & $-0.017$ & 1024 & 32.010 & $+0.064$ \\
\midrule
\egsrow RL & 4096 & 30.956 & $+0.024$ & \multicolumn{3}{c}{\emph{capped set ($\kappa{=}M_{\max}$)}}\\
\egsrow RL & 2048 & 30.557 & $-0.375$ & \multicolumn{3}{c}{\emph{ranks below cap}}\\
\egsrow RL & 1024 & 29.805 & $-1.127$ & \multicolumn{3}{c}{\emph{ranks below cap}}\\
\bottomrule
\end{tabular}
\end{table}

\noindent\textbf{Trimming equivalence test.} \cref{tab:supp_trim_tost} runs the paired
TOST of \cref{sec:supp_bootstrap} on \textit{Trimming} at full refinement, with
the scene-specific block length $L{=}50$. Every deployment-budget selector is
statistically equivalent to FPS@8192 within the sub-perceptual margin
$\delta{=}0.25$~dB. The lone non-equivalent entry is FPS@4096, which clears the
margin only by being \emph{better} ($+0.27$~dB); this is the budget non-monotonicity also
visible in \cref{tab:supp_frontier}, not a sign that selection matters.

\begin{table}[t]
\centering
\caption{\textbf{Trimming equivalence test (TOST).} Paired moving-block CIs vs.\
the FPS@8192 reference at full refinement ($r{=}50$); $L{=}50$ (set from
Trimming's IACT${\approx}47$), $90\%$ interval (the TOST interval for SESOI
$\delta{=}0.25$~dB). ``equiv'' holds when the interval lies within $\pm\delta$.
The sole non-equivalent entry, FPS@4096, exceeds the margin on the \emph{better}
side ($+0.27$~dB), a budget effect rather than selection sensitivity.}
\label{tab:supp_trim_tost}
\scriptsize
\setlength{\tabcolsep}{4.0pt}
\renewcommand{\arraystretch}{1.05}
\begin{tabular}{l c c c}
\toprule
\textbf{Method} & $\kappa$ & $\Delta$ \textbf{vs FPS@8192 (90\% CI)} & \textbf{equiv} \\
\midrule
FPS        & 4096 & $+0.268$ $[+0.164,+0.351]$ & no (better) \\
Random     & 4096 & $-0.064$ $[-0.188,+0.070]$ & yes \\
Uniform    & 4096 & $-0.052$ $[-0.112,+0.000]$ & yes \\
Import.    & 4096 & $+0.019$ $[-0.107,+0.150]$ & yes \\
Random     & 8192 & $+0.048$ $[-0.047,+0.154]$ & yes \\
Uniform    & 8192 & $+0.026$ $[-0.112,+0.157]$ & yes \\
Import.    & 8192 & $+0.050$ $[-0.101,+0.151]$ & yes \\
\bottomrule
\end{tabular}
\end{table}

% =====================================================================
\section{N3DV Results}
\label{sec:supp_n3dv}

We report two N3DV checks. \cref{tab:supp_n3dv_grid} gives the four-scene
relative spread used for the grid analysis. \cref{tab:supp_n3dv_frontier} gives
the latest held-out validation-set metrics used in the main table.
For the validation-set table, RL uses the N3DV checkpoint at each reported
budget. The four-scene spread over the standard selectors has median
$0.06$~dB, matching MeetingRoom's \textit{Discussion}.
The one consistently larger column is \textit{Coffee Martini} at tight budgets,
where FPS runs slightly high; it tends to shrink with budget (non-monotonic on Coffee Martini).

\begin{table}[t]
\centering
\caption{\textbf{N3DV four-scene cross-selector PSNR spread (dB).}
Best-minus-worst over FPS/random/uniform/importance at each budget; $31$ held-out
frames per scene. Median over the four deployment budgets ($1024$--$8192$) and
four scenes $=0.06$~dB.}
\label{tab:supp_n3dv_grid}
\scriptsize
\setlength{\tabcolsep}{4.5pt}
\renewcommand{\arraystretch}{1.05}
\begin{tabular}{l cccc}
\toprule
$\kappa$ & \textit{Coffee} & \textit{Cook} & \textit{Salmon} & \textit{Steak} \\
\midrule
$1024$  & 0.294 & 0.037 & 0.155 & 0.046 \\
$2048$  & 0.357 & 0.057 & 0.062 & 0.070 \\
$4096$  & 0.251 & 0.022 & 0.090 & 0.049 \\
$8192$  & 0.186 & 0.030 & 0.057 & 0.046 \\
$16384$ & 0.091 & 0.016 & 0.056 & 0.026 \\
\bottomrule
\end{tabular}
\end{table}

\cref{tab:supp_n3dv_frontier} matches the N3DV validation-set block in the main
paper.

\begin{table}[t]
\centering
\caption{\textbf{N3DV validation-set frontier.} PSNR, LPIPS, and DSSIM for the
same held-out validation set reported in the main paper. RL rows use the N3DV
checkpoint at the reported budgets.}
\label{tab:supp_n3dv_frontier}
\scriptsize
\setlength{\tabcolsep}{4.0pt}
\renewcommand{\arraystretch}{1.04}
\begin{tabular}{l c c c c}
\toprule
\textbf{Method} & $\kappa$ & PSNR$\uparrow$ & LPIPS$\downarrow$ & DSSIM$\downarrow$ \\
\midrule
FPS    & 8192 & 31.698 & 0.1070 & 0.02309 \\
FPS    & 4096 & 31.702 & 0.1068 & 0.02297 \\
FPS    & 2048 & 31.646 & 0.1066 & 0.02298 \\
FPS    & 1024 & 31.699 & 0.1063 & 0.02302 \\
\midrule
Random & 8192 & 31.792 & 0.1056 & 0.02285 \\
Random & 4096 & 31.713 & 0.1060 & 0.02298 \\
Random & 2048 & 31.830 & 0.1057 & 0.02270 \\
Random & 1024 & 31.669 & 0.1061 & 0.02298 \\
\midrule
Uniform & 8192 & 31.651 & 0.1063 & 0.02302 \\
Uniform & 4096 & 31.692 & 0.1065 & 0.02295 \\
Uniform & 2048 & 31.845 & 0.1058 & 0.02282 \\
Uniform & 1024 & 31.760 & 0.1059 & 0.02282 \\
\midrule
Import. & 8192 & 31.582 & 0.1072 & 0.02322 \\
Import. & 4096 & 31.620 & 0.1065 & 0.02302 \\
Import. & 2048 & 31.614 & 0.1068 & 0.02307 \\
Import. & 1024 & 31.684 & 0.1062 & 0.02295 \\
\midrule
\egsrow RL & 8192 & 31.447 & 0.1089 & 0.02368 \\
\egsrow RL & 4096 & 32.501 & 0.0979 & 0.02121 \\
\egsrow RL & 2048 & 33.675 & 0.0860 & 0.01928 \\
\egsrow RL & 1024 & 33.846 & 0.0843 & 0.01891 \\
\bottomrule
\end{tabular}
\end{table}

% =====================================================================
\section{Cross-Backbone Detail: 3DGStream}
\label{sec:supp_3dgs}

This section expands the second-backbone check of the main paper. On
3DGStream, which advances the scene by gradient-driven densification through a
neural transform cache rather than by anchor-blended skinning, ``selection''
is which Gaussians the per-frame update acts on. We replace the native
gradient-magnitude rule with random and FPS at a matched budget under one code
path, exactly as on IGS, and ask whether the selection rule changes quality on a
mechanism unlike LBS, a test of whether \cref{prop:invariance}'s coverage
argument is about density rather than the specific update rule.

\cref{tab:supp_3dgs} reports the per-rule means and paired moving-block intervals. At quality in 3DGStream's normal N3DV range, the three selection rules differ by at most $0.02$~dB and every paired $95\%$ interval contains zero, the same selection invariance we measure on IGS, now on a backbone whose update mechanism shares nothing with LBS.

\begin{table}[t]
\centering
\caption{\textbf{Cross-backbone selector comparison (3DGStream).} Mean PSNR over the $50$ matched streamed frames of one N3DV scene (filter-corrected initialization, normal quality range), matched budget, single code path. $\Delta$ and paired $95\%$ CI are vs.\ the native gradient rule (moving-block bootstrap, $L{=}10$). A confirmatory single-scene check, not a second grid.}
\label{tab:supp_3dgs}
\footnotesize
\setlength{\tabcolsep}{5pt}
\renewcommand{\arraystretch}{1.05}
\begin{tabular}{l c c c}
\toprule
\textbf{Selection rule} & \textbf{PSNR} $\uparrow$ & $\Delta$ vs.\ grad & \textbf{95\% CI} \\
\midrule
Gradient (native) & 32.69 & \;\;\;0.00 & n/a \\
Random            & 32.70 & $+0.01$ & $[-0.04,+0.04]$ \\
FPS               & 32.71 & $+0.02$ & $[-0.02,+0.06]$ \\
\midrule
\multicolumn{4}{l}{best$-$worst spread $=0.02$~dB; all intervals contain $0$} \\
\bottomrule
\end{tabular}
\end{table}

% =====================================================================
\section{Moving-Block Bootstrap: Full Procedure}
\label{sec:supp_bootstrap}

All intervals use the moving-block bootstrap of \cref{alg:bootstrap}: we resample
contiguous blocks rather than individual frames because adjacent streamed frames
are strongly correlated, and an i.i.d.\ frame bootstrap understates the interval.

\noindent\textbf{Block length from autocorrelation.} We set $L$ from the integrated
autocorrelation time (IACT) of per-frame PSNR, estimated from the empirical
autocorrelation function: IACT${\approx}6$ frames on \textit{Discussion}
(autocorrelation drops below $0.1$ by lag $7$) and ${\approx}47$ on the more
slowly varying \textit{Trimming}, so we use $L{=}15$ and $L{=}50$ respectively,
both comfortably above the correlation length. The choice is not load-bearing:
the $95\%$ gap-CI width on \textit{Discussion} is flat for $L\in[10,50]$ (within
$0.29\pm0.01$~dB), and every equivalence verdict is unchanged at $L{=}50$.

\noindent\textbf{Paired test and equivalence (TOST).} For a pair we resample the
\emph{paired} per-frame difference on the frames the two samplers share, which
cancels per-frame scene difficulty and tightens the gap interval. Rather than
read non-significance off a zero-containing interval, we run two one-sided tests
against a sub-perceptual margin $\delta{=}0.25$~dB: two samplers are
\emph{equivalent} when the bootstrap gap interval lies within $\pm\delta$. At
deployment budgets every cheap sampler (random, uniform, importance), and the
MeetingRoom learned-policy parity point, is equivalent to FPS@8192 by this test on
both scenes (\textit{Trimming} in \cref{tab:supp_trim_tost}); the only
non-equivalent deployment-budget entry is FPS@4096 on \textit{Trimming}, which is
\emph{better} than FPS@8192.

\noindent\textbf{Noise floor for the grid.} The four-sampler best-minus-worst spread of
\cref{tab:supp_grid} is a positively biased order statistic, nonzero even under a
true null. We bootstrap it under the null, four independent estimates of a single
sampler, obtaining a median spread of $0.32$~dB on \textit{Discussion} and
$0.86$~dB on \textit{Trimming} (the latter larger because of Trimming's stronger
autocorrelation and hence fewer effective frames). Rather than compare to the
floor's median alone, we locate the observed spread in the null distribution: it
falls at the $6$th percentile on \textit{Discussion} and the $2$nd on
\textit{Trimming}, i.e.\ pure measurement noise produces a spread at least as
large as observed in $94\%$ and $98\%$ of resamples. The cross-sampler variation
is thus statistically within what measurement noise alone produces.

% =====================================================================
\section{Selector Forward-Pass Latency}
\label{sec:supp_latency}
\cref{tab:fwd_latency} reports the full forward-pass latency sweep summarized by the latency figure in the main paper. The learned selector is exported to TensorRT (CUDA graph, deployed engine); the figure quoted there is the larger-cap ($M_{\max}{=}16{,}384$) export used as a scaling microbenchmark, since the deployed $M_{\max}{=}4096$ checkpoint cannot emit $\kappa{>}4096$. The TensorRT forward is pool-independent (it depends only on the candidate count, capped at $16{,}384$); both FPS variants grow with budget and pool. These are forward-pass numbers only, not the in-pipeline selection stage (which adds feature extraction).

\begin{table}[t]
\centering
\caption{\textbf{Selector forward-pass latency} (A100, $100$k candidate pool). RL-TRT is the TensorRT export; FPS-CPU is IGS's bucketed kd-line FPS; FPS-GPU is exact \texttt{torch\_cluster.fps}. Speedups are RL-TRT relative to each FPS variant. The exact GPU FPS is slower than the optimized CPU routine because farthest-point selection is a sequential recurrence that GPU parallelism does not accelerate.}
\label{tab:fwd_latency}
\scriptsize
\setlength{\tabcolsep}{3.4pt}
\renewcommand{\arraystretch}{1.06}
\begin{tabular}{c c c c c c}
\toprule
\textbf{Anch.} & \textbf{RL-TRT} & \textbf{FPS-CPU} & \textbf{FPS-GPU} & \textbf{Spd. vs} & \textbf{Spd. vs} \\
$\kappa$ & \textbf{(ms)} $\downarrow$ & \textbf{(ms)} & \textbf{(ms)} & \textbf{CPU} & \textbf{GPU} \\
\midrule
1024  & \textbf{0.185} & 29.0  & 204.3  & 157$\times$ & 1104$\times$ \\
2048  & \textbf{0.210} & 42.8  & 408.6  & 204$\times$ & 1946$\times$ \\
4096  & \textbf{0.341} & 67.0  & 817.3  & 196$\times$ & 2397$\times$ \\
8192  & \textbf{0.834} & 109.4 & 1634.7 & 131$\times$ & 1960$\times$ \\
16384 & \textbf{2.742} & 186.1 & 3269.2 & 68$\times$  & 1192$\times$ \\
\bottomrule
\end{tabular}
\end{table}

% =====================================================================
\section{Training and Evaluation Hyperparameters}
\label{sec:supp_hyperparams}
\cref{tab:supp_hyperparams} lists the values used for the learned policy and the
strict high-quality validation path. These are provided for reproducibility; the
learned policy is analyzed, not proposed, in the main paper.

\begin{table}[t]
\centering
\caption{\textbf{Learned-policy and high-quality refinement settings.}}
\label{tab:supp_hyperparams}
\scriptsize
\setlength{\tabcolsep}{3.0pt}
\renewcommand{\arraystretch}{1.04}
\begin{tabular}{p{0.44\linewidth} p{0.48\linewidth}}
\toprule
\textbf{Parameter} & \textbf{Value} \\
\midrule
Candidate cap $M_{\max}$ & 4096 \\
Deployment buckets & 512, 1024, 2048, 4096 \\
Point descriptor dimension & 13 \\
Transformer layers / heads / dim & 2 / 4 / 128 \\
Dropout & 0.1 \\
Visual feature views & 2 \\
Policy epochs / steps & 4 / 1600 \\
Max batches per epoch & 400 \\
Sampler learning rate & $10^{-4}$ \\
Weight decay / gradient clip & 0.01 / 1.0 \\
Teacher anchor size & 4096 \\
Warm-start epochs & 0 \\
$\delta$ PSNR tolerance & 0.12 \\
$\lambda_{\kappa}$ / $\lambda_t$ / $\lambda_v$ / $\lambda_g$ & 0.20 / 0.60 / 8.0 / 0.25 \\
Entropy coefficient $\beta$ & 0.01 \\
HQ refinement iterations & 10 \\
HQ refine / score views & 2 / 1 \\
Densification & enabled \\
Densify interval / threshold & 5 / 0.00015 \\
Max refined Gaussians & 40000 \\
\bottomrule
\end{tabular}
\end{table}

% =====================================================================
\section{Qualitative Diagnostics}
\label{sec:supp_qualitative}
These panels visualize the central finding on \emph{real} reconstructions across all three datasets. The blink tests (\cref{fig:supp_blink,fig:supp_blink_disc,fig:supp_blink_n3dv}) show all five samplers' renders of one held-out key-frame side by side, visually indistinguishable on \textit{Trimming}, \textit{Discussion}, and N3DV alike, even though the samplers place anchors completely differently and impose very different per-anchor load (quantified in the main paper's back-projection figure). \cref{fig:supp_effect} confirms the resulting output difference is indistinguishable from reconstruction noise. All renders are at budget $4{,}096$.

\begin{figure*}[t]
  \centering
  \begin{subfigure}{\linewidth}\centering
    \includegraphics[width=\linewidth]{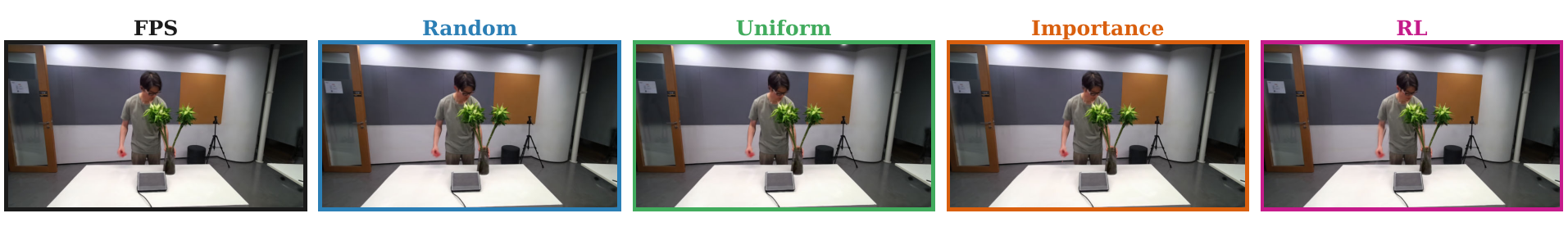}
    \caption{MeetingRoom \textit{Trimming} (key-frame $46$).}\label{fig:supp_blink}
  \end{subfigure}\\[4pt]
  \begin{subfigure}{\linewidth}\centering
    \includegraphics[width=\linewidth]{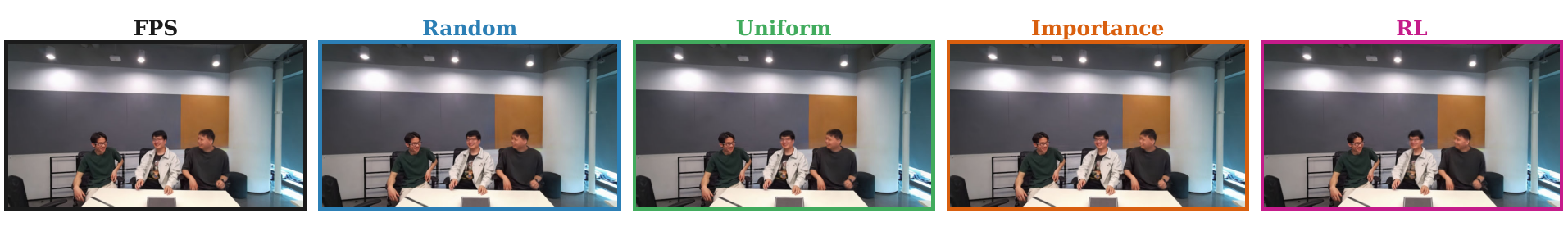}
    \caption{MeetingRoom \textit{Discussion} (key-frame $36$).}\label{fig:supp_blink_disc}
  \end{subfigure}\\[4pt]
  \begin{subfigure}{\linewidth}\centering
    \includegraphics[width=\linewidth]{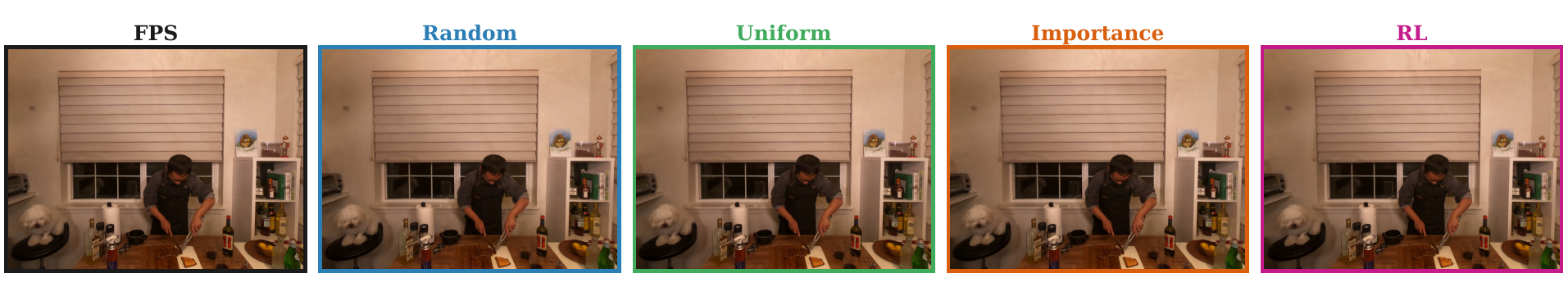}
    \caption{N3DV \textit{cut\_roasted\_beef} (key-frame $6$).}\label{fig:supp_blink_n3dv}
  \end{subfigure}
  \caption{\textbf{The blink test, all three datasets} (budget $4{,}096$). All five samplers' renders of one held-out key-frame, side by side. The reconstructions are visually indistinguishable on every scene, though the samplers place anchors completely differently; the statistical equivalence is established in the main paper. Per-frame PSNRs are omitted (a single held-out camera does not reflect the all-frames equivalence).}
  \label{fig:supp_blink_all}
\end{figure*}

\begin{figure*}[p]
  \centering
  \includegraphics[width=\linewidth,height=0.93\textheight,keepaspectratio]{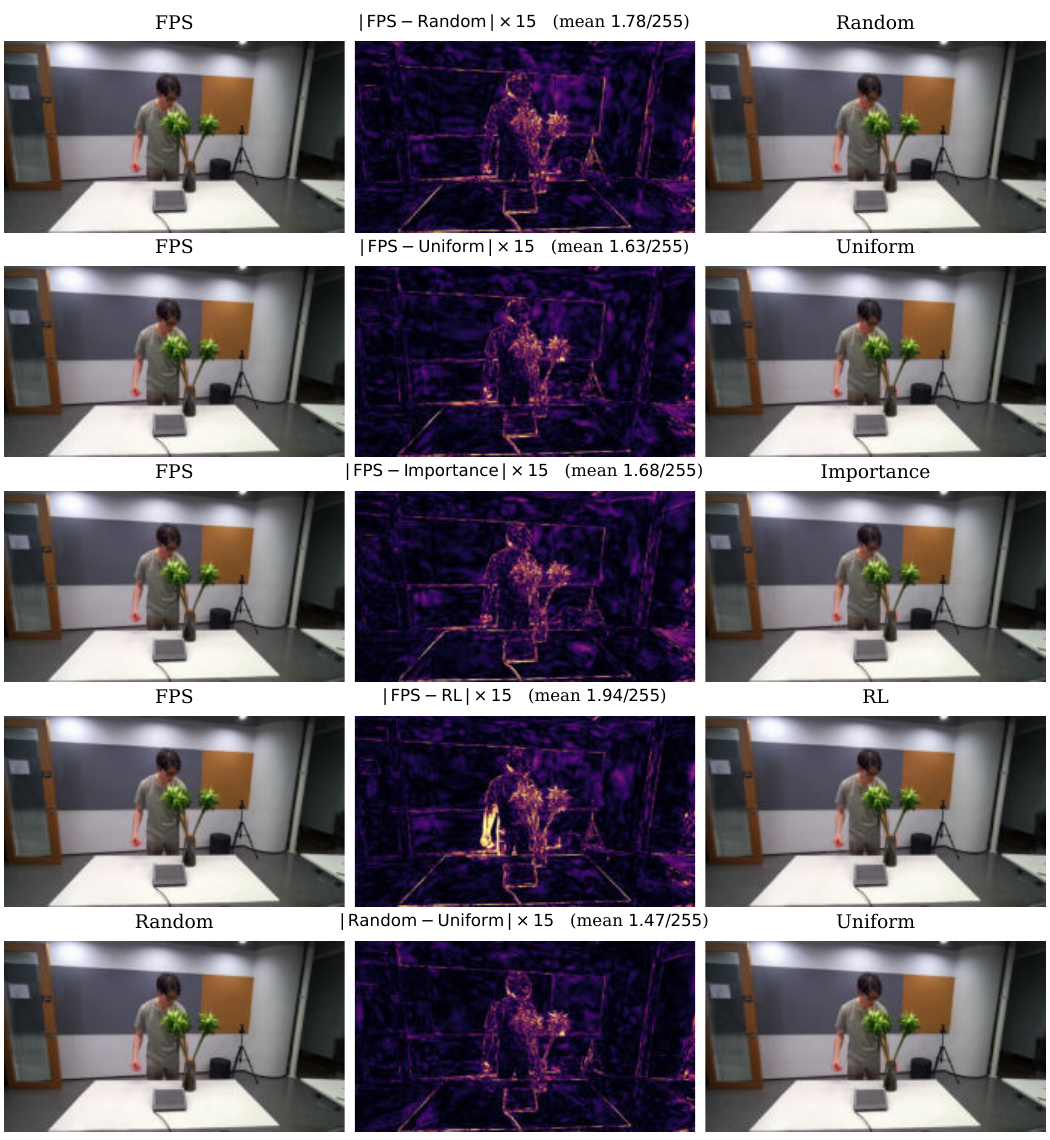}
  \caption{\textbf{Pairwise output differences are indistinguishable from noise (1 of 2).} For five of the ten selector pairs, each row is [\,render A $\mid$ $|\text{A}-\text{B}|\times15$ $\mid$ render B\,] on the held-out \textit{Trimming} key-frame $46$ (budget $4{,}096$; \texttt{inferno}, $\times15$ amplification). Every difference is near-empty, only faint high-frequency reconstruction noise along edges. Mean $|\Delta|$ stays in $1.4$--$1.9/255$ across pairs.}
  \label{fig:supp_effect}
\end{figure*}

\begin{figure*}[p]
  \centering
  \includegraphics[width=\linewidth,height=0.93\textheight,keepaspectratio]{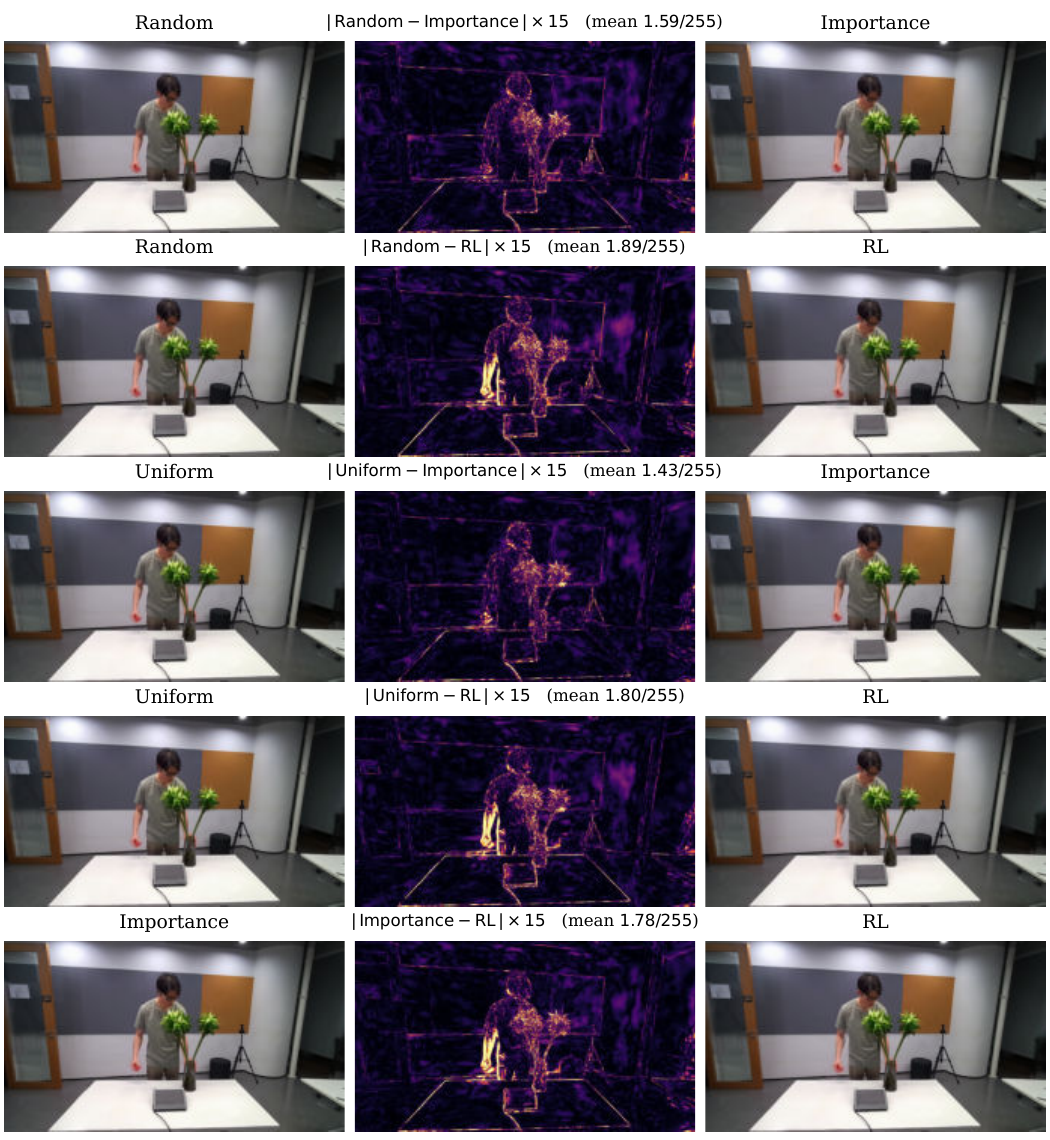}
  \caption{\textbf{Pairwise output differences are indistinguishable from noise (2 of 2).} The remaining five selector pairs (same layout and scaling as \cref{fig:supp_effect}). The largest mean difference across all ten pairs is $1.94/255$ (FPS--RL) and the smallest $1.43/255$ (Uniform--Importance); even the learned policy versus the others differs only by noise. The perceptual metrics tie accordingly (LPIPS $\le 0.0007$, DSSIM $\le 0.0004$).}
  \label{fig:supp_effect2}
\end{figure*}

% =====================================================================
\section{Back-Projection Galleries}
\label{sec:supp_backproj}
For each scene we show two held-out key-frames. Within a key-frame the rows are, top to bottom: the refined render under each selector (visually identical); the back-projected \emph{selected anchors} as points colored by load; and the smooth $K{=}8$ LBS \emph{load field} (the main paper's Fig.~3 visualization). Columns are the five selectors. The renders coincide while the anchor placement and per-anchor load differ sharply, the learned policy and the opacity-scale heuristic concentrating load on a few hot control points and random/uniform spreading it. Mean load is fixed by $K\times(\#\text{candidates})/\kappa$ and is equal across selectors, so the contrast is in the \emph{peak} load (shown per panel).

\begin{figure*}[p]
  \centering
  \includegraphics[width=\linewidth,height=0.92\textheight,keepaspectratio]{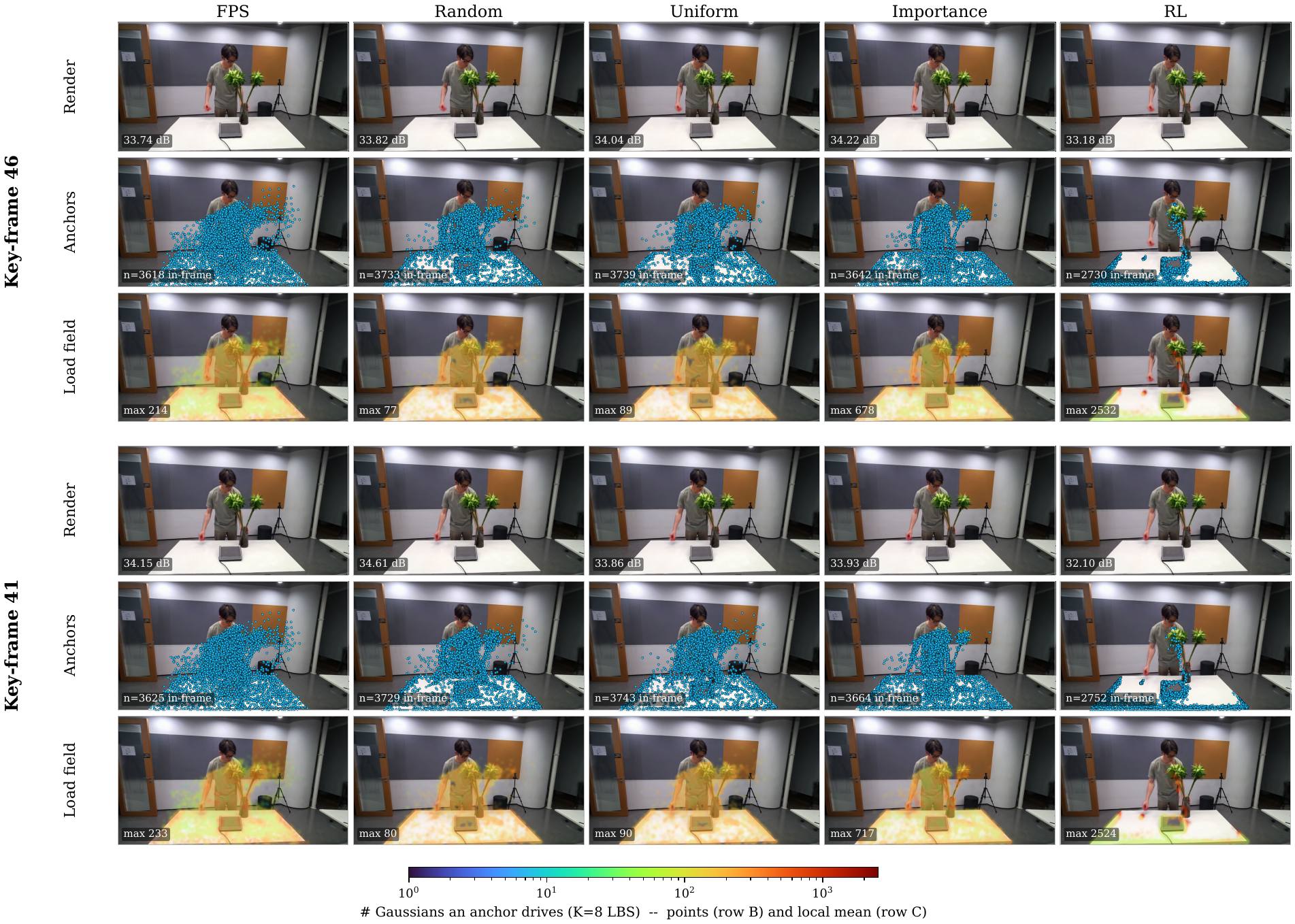}
  \caption{\textbf{Back-projection gallery: MeetingRoom \textit{Trimming}} (key-frames $46$ and $41$). Render $\rightarrow$ selected anchors $\rightarrow$ smooth load field, for all five selectors. The reconstructions are indistinguishable; the load distributions are not.}
  \label{fig:supp_bp_trimming}
\end{figure*}

\begin{figure*}[p]
  \centering
  \includegraphics[width=\linewidth,height=0.92\textheight,keepaspectratio]{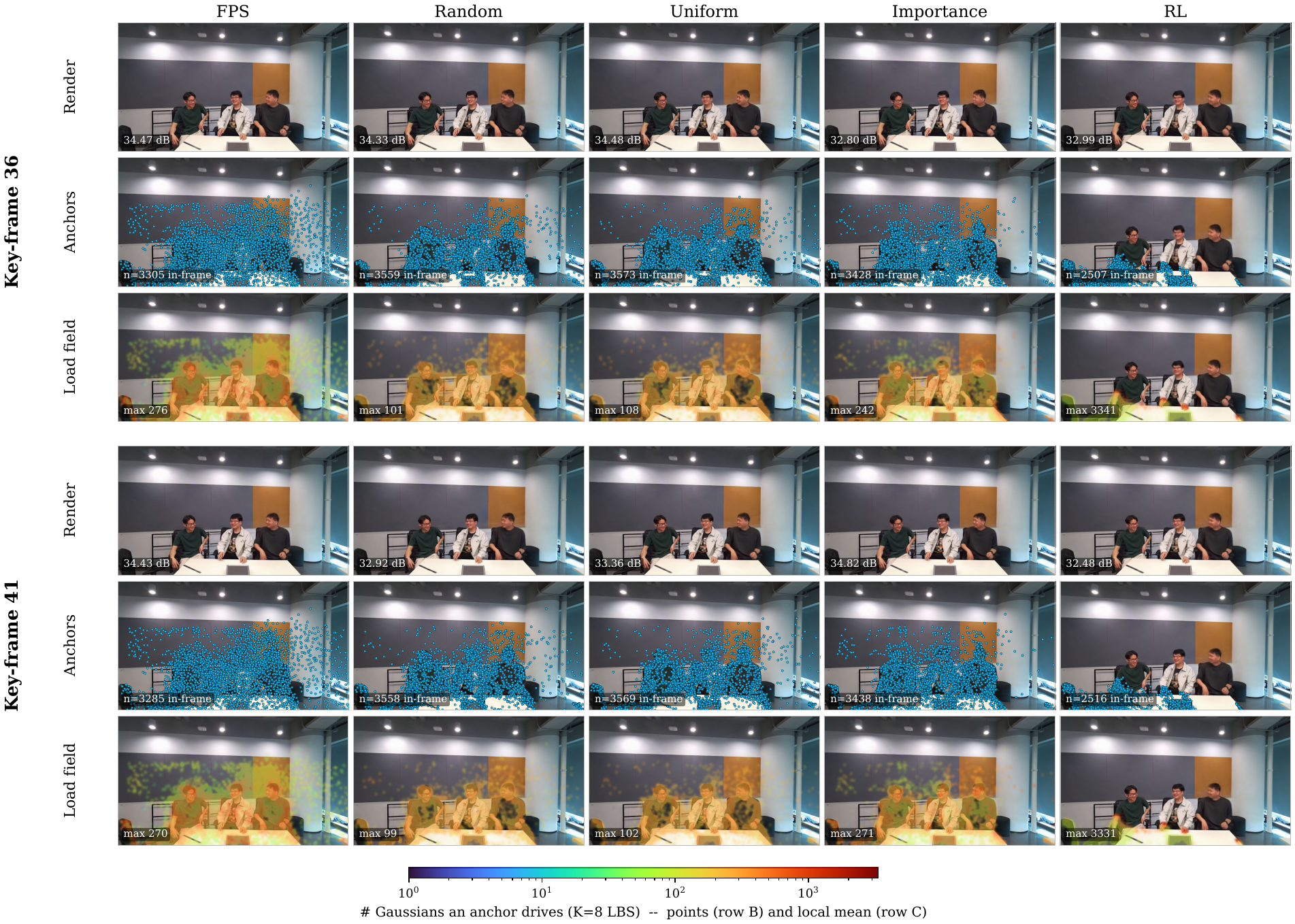}
  \caption{\textbf{Back-projection gallery: MeetingRoom \textit{Discussion}} (key-frames $36$ and $41$). Same layout as \cref{fig:supp_bp_trimming}; the cross-selector invariance and the concentrate-versus-spread load contrast both persist on the second MeetingRoom scene.}
  \label{fig:supp_bp_discussion}
\end{figure*}

\begin{figure*}[p]
  \centering
  \includegraphics[width=\linewidth,height=0.92\textheight,keepaspectratio]{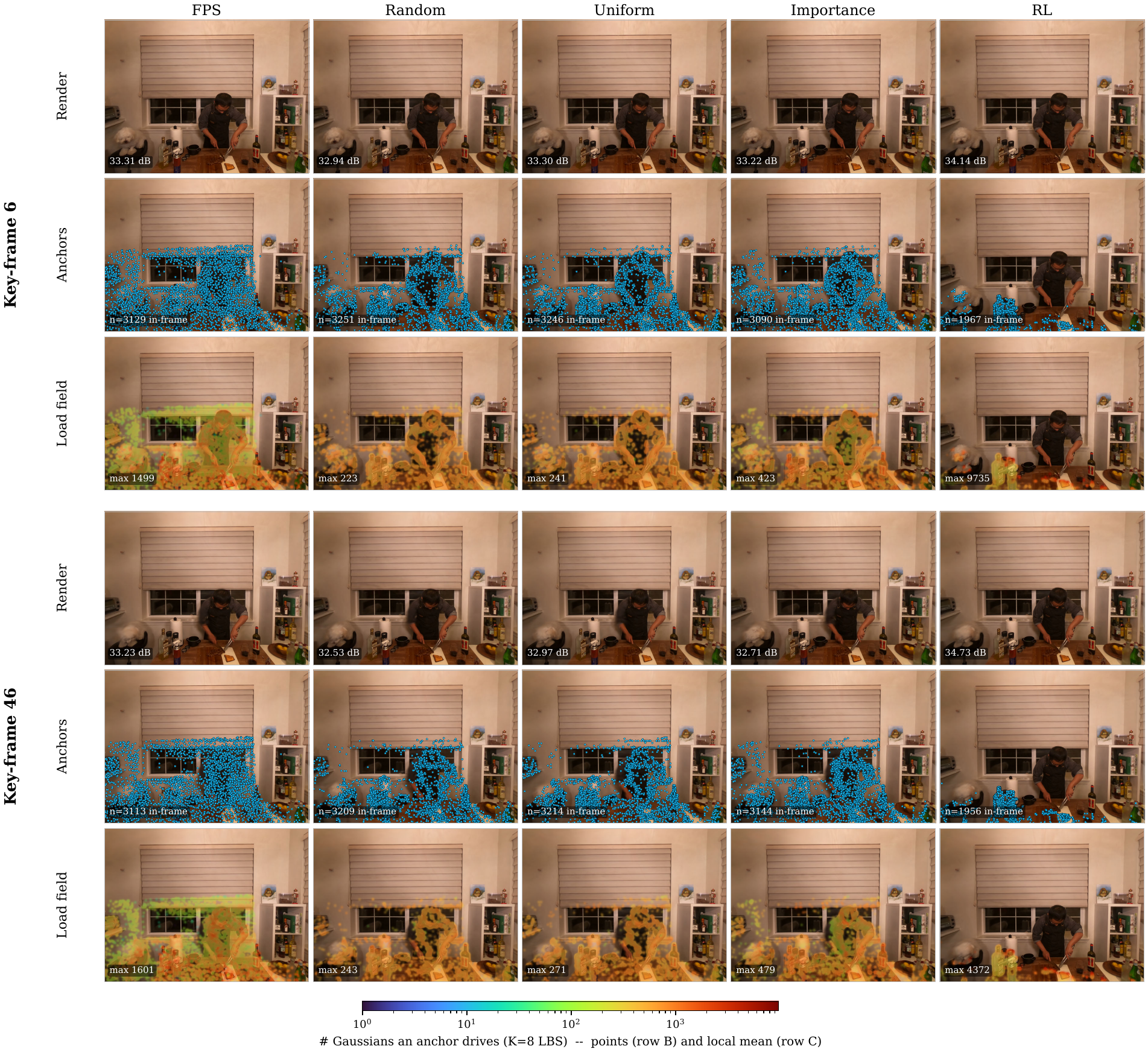}
  \caption{\textbf{Back-projection gallery: N3DV (\textit{cut\_roasted\_beef})} (key-frames $6$ and $46$). Same layout as \cref{fig:supp_bp_trimming}. On N3DV the load is higher overall (denser candidate pool) and the learned policy concentrates most aggressively (peak ${>}9{,}000$ Gaussians driven by one anchor), yet the renders remain indistinguishable.}
  \label{fig:supp_bp_n3dv}
\end{figure*}